\documentclass[ twocolumn]{autart} 

\usepackage{amsmath,amssymb,amsfonts}
\usepackage{mathtools}
\usepackage{float}
\usepackage{caption}
\usepackage{graphicx,subcaption}
\usepackage{hyperref}
\usepackage{enumitem}
\usepackage{makecell}
\usepackage{changepage}
\usepackage{centernot}
\usepackage[round]{natbib}

\usepackage{algorithm}
\usepackage{algpseudocode}
\algrenewcommand\algorithmicindent{0.8em}

\usepackage{pgfplots}
\usepackage{tikz-network}
\DeclareUnicodeCharacter{2212}{−}
\usepgfplotslibrary{groupplots,dateplot}
\usetikzlibrary{patterns,shapes.arrows}
\pgfplotsset{compat=newest}

\usepackage{pifont}

\DeclareMathAlphabet{\mymathbb}{U}{BOONDOX-ds}{m}{n}

\hypersetup{%
  colorlinks=true,
  urlbordercolor={1 1 1},
  linkcolor={black},
  raiselinks=false,
  citecolor=black
}
\usepackage{etoolbox}

\makeatletter
\def\@xnamedef#1{\expandafter\protected@xdef\csname #1\endcsname}
\def\no@harm{} 
\def\ead@au#1{\protected@edef\@ead@au{#1}}
\patchcmd\runningauthor@fmt{\global\edef}{\protected@xdef}{}{}
\patchcmd\runningauthor@fmt{\global\edef}{\protected@xdef}{}{}
\patchcmd\author@fmt{\edef}{\protected@edef}{}{}
\patchcmd\add@xtok{\xdef}{\protected@xdef}{}{}
\makeatother

\makeatletter
\DeclareRobustCommand{\qed}{%
  \ifmmode 
  \else \leavevmode\unskip\penalty9999 \hbox{}\nobreak\hfill
  \fi
  \quad\hbox{\qedsymbol}}
\newcommand{\openbox}{\leavevmode
  \hbox to.77778em{%
  \hfil\vrule
  \vbox to.675em{\hrule width.6em\vfil\hrule}%
  \vrule\hfil}}
\newcommand{\qedsymbol}{\openbox}
\newenvironment{proof}[1][\proofname]{\par
  \normalfont
  \topsep6\p@\@plus6\p@ \trivlist
  \item[\hskip\labelsep\itshape
    #1.]\ignorespaces
}{%
  \qed\endtrivlist
}
\newcommand{\proofname}{Proof}
\makeatother

\newcounter{thms}
\newtheorem{theorem}[thms]{Theorem}
\newcounter{lems}
\newtheorem{lemma}[lems]{Lemma}
\newcounter{cors}
\newtheorem{corollary}[cors]{Corollary}

\newtheorem{proposition}[thms]{Proposition}

\newcounter{defsex}
\newtheorem{definition}[defsex]{Definition}
\newtheorem{example}[defsex]{Example}
\newcounter{asses}
\newtheorem{assumption}[asses]{Assumption}

\newcommand{\symi}{h}
\newcommand{\symk}{j}
\newcommand{\symt}{t}
\newcommand{\symkit}{k}

\newcommand{\symc}{c}
\newcommand{\symC}{C}
\newcommand{\symf}{f}
\newcommand{\symF}{F}

\newcommand{\symh}{i}
\newcommand{\symj}{\ell}
\newcommand{\symhb}{\symh}

\newcommand{\R}{\mathbb{R}}

\newcommand{\bb}[1]{\mymathbb{#1}}

\newcommand{\mc}[1]{\mathcal{#1}}
\newcommand{\mt}[1]{\textrm{#1}}

\newcommand{\lt}{\left}
\newcommand{\rt}{\right}
\newcommand{\beeq}{\begin{equation}}
\newcommand{\eneq}{\end{equation}}
\newcommand{\matb}{\begin{matrix}}
\newcommand{\mate}{\end{matrix}}

\newcommand{\unaryminus}{\scalebox{0.6}[1.0]{$ - $}}

\DeclareMathOperator*{\argmax}{arg\,max}
\DeclareMathOperator*{\argmin}{arg\,min}

\definecolor{fillred}{RGB}{242, 210, 208}
\definecolor{fillgreen}{RGB}{216, 232, 210}
\definecolor{fillblue}{RGB}{210, 223, 236}
\definecolor{fillorange}{RGB}{251, 227, 208}

\definecolor{redtmp}{rgb}{0, 0, 0}
\definecolor{redrev1}{rgb}{0, 0, 0}

\begin{document}

\begin{frontmatter}

\title{
Maximum likelihood inference for high-dimensional problems with multiaffine variable relations\thanksref{footnoteinfo}
}

\thanks[footnoteinfo]{This research is supported by the Swiss National Science Foundation under the NCCR Automation (grant agreement 51NF40\_180545).}

\author[epfl]{Jean-S\'{e}bastien Brouillon}\ead{jean-sebastien.brouillon@epfl.ch},    
\author[ethz]{Florian D\"{o}rfler}\ead{dorfler@control.ee.ethz.ch},               
\author[epfl]{Giancarlo Ferrari-Trecate}\ead{ giancarlo.ferraritrecate@epfl.ch}  

\address[epfl]{Institute of Mechanical Engineering, \'Ecole Polytechnique F\'ed\'erale de Lausanne, Switzerland}  
\address[ethz]{Automatic Control Laboratory, Swiss Federal Institute of Technology (ETH), Switzerland}



\begin{abstract}
Maximum Likelihood Estimation of continuous variable models can be very challenging in high dimensions, due to potentially complex probability distributions. The existence of multiple interdependencies among variables can make it very difficult to establish convergence guarantees. This leads to a wide use of brute-force methods, such as grid searching and Monte-Carlo sampling and, when applicable, complex and problem-specific algorithms. In this paper, we consider inference problems where the variables are related by multiaffine expressions. We propose a novel Alternating and Iteratively-Reweighted Least Squares (AIRLS) algorithm, and prove its convergence \textcolor{redrev1}{for problems with Generalized Normal Distributions}. We also provide an efficient method to compute the variance of the estimates obtained using AIRLS. Finally, we show how the method can be applied to graphical statistical models. We perform numerical experiments on several inference problems, showing significantly better performance than state-of-the-art approaches in terms of scalability, robustness to noise, and convergence speed due to an empirically observed super-linear convergence rate.
\end{abstract}

\end{frontmatter}


\section{Introduction} \label{section_intro}
Statistical inference is widely used in many disciplines such as environmental sciences, economics, energy systems, and control theory \citep{survey_environment_bns, water_bn, econ_iv_graph, survey_energy_bns, control_building_bn}. More precisely, inference allows learning and predicting model variables from noisy observations, sometimes also providing a measure of the predictive uncertainty. However, fitting a model to data generally leads to complex inference problems. Furthermore, specific variables of interest often need to be estimated although some latent (unobserved) variables are unknown. 

Among inference methods, Maximum Likelihood Estimation (MLE) is very popular for its consistency and efficiency properties \citep{statistical_inference_book, mle_consistency}. The computation of MLE estimates can be fairly simple for problems where few variables need to be inferred, but becomes much harder when their numbers increase \citep{proba_ml_book}. Rather than inferring the value of latent variables, one can marginalize the likelihood over them, like in Expectation-Maximization (EM). Marginalizations may improve the estimation accuracy, but are often not tractable when the latent variables are high-dimensional or do not take discrete values \citep{tractable_em}. 
 
In MLE, if the probability distributions of all continuous variables are Gaussian, the likelihood can often be maximized using standard convex optimization techniques such as Gradient Descent (GD) and ADMM \citep{admm_def}. Additionally, several methods have been developed for problems where the relations between the Gaussian random variables follow complex graph structures \citep{gaussian_learn_rain, gaussian_learn_mixed, gaussian_learn_design}. \textcolor{redrev1}{More generally, proximal methods and Iteratively Reweighted Least Squares (IRLS) provide a simple adaptation to Generalized Normal Distributions (GNDs, also called exponential power distributions), but only in specific cases \citep{beck2017_prox_methods}. Moreover, current methods focus on problems where all GNDs have the same exponent, which amount to minimizing a norm.}

Many inference problems contain multiaffine relations between \textcolor{redrev1}{GND-distributed random} variables. For example, the unknown parameters in both Error-In-Variables (EIV) and rank-constrained tensor regression models relate in a multi-linear way \citep{tensor_eiv, tensor_imaging_noeiv}. Signal processing and system identification problems can also present multiaffine relations, e.g., when unknown filter or system parameters multiply with unmeasured disturbances \citep{sysid_ml_multilin, soderstrom_book}. \textcolor{redrev1}{Additionally, \cite{lq_norm_optimization} have shown that adapting the exponent of GNDs to a problem can substantially improve the goodness of fit.}

The aforementioned challenges explain why almost all practical MLE methods for complex problems are based on Monte Carlo sampling and discretization (i.e., grid search) \citep{discretization_def, discretization_multi,sampling_book}. Although these two approaches are simple to implement and provide an approximate probability distribution for the estimates, they are dependent on the sampling distributions or discretization strategies used, which may impact the estimation accuracy \citep{discretization_def}. The computational complexity of these methods, growing rapidly with both the desired accuracy and the dimension of the problem, is the most important limitation, which prevents their use in many applications. Other zeroth-order methods such as the Gradient-Less Descent (GLD) can help to reduce the computational burden in high dimensions, but display slow convergence \citep{zero_order_gd}.



The contributions of this paper are fourfold. First, we present an efficient algorithm, called Alternating and Iteratively-Reweighted Least Squares (AIRLS), for MLE with multiaffine-related variables. \textcolor{redrev1}{Second, we prove the convergence of AIRLS if each variable follows a GND. Moreover, we discuss the optimality of the likelihood of the estimate for finite-precision solvers, showing that the suboptimality decays with the solver's precision.} Third, we describe a method to obtain the confidence intervals of the MLE estimates. This last step is particularly useful because, while MLE provides estimates, it is often hard to compute their precision given by the Cramer-Rao lower bound \citep{statistical_inference_book}. This bound depends on the likelihood evaluated at the ground truth, and can be very sensitive to uncertainty when using estimates instead of (unknown) exact values. Fourth and finally, all the contributions are substantiated by extensive simulations on various problems from different scientific fields. \textcolor{redrev1}{Empirical evidence shows that AIRLS converges to a meaningful estimate even when the distributions are not restricted to GNDs.} The main algorithm is implemented in a custom developed open source package \citep{python_package}, which provides an easy-to-use interface for MLE problems with multiaffine related variables faster than with other aforementioned methods. Application-specific versions of the AIRLS algorithm were presented in \cite{paper_tsg} and \cite{cdc_paper_recursive}. Compared to the previous works, this paper considers a much wider class of likelihoods, presents a deeper theoretical analysis of the algorithm's convergence, and provide tools to evaluate the quality of the final estimates.

The paper is organized as follows. Section \ref{section_prob} states the inference problem and Section \ref{section_method} presents our algorithm. Section \ref{section_suboptimal} discusses the optimality of the result's likelihood. Section \ref{sec_variance} shows how to compute the estimate's variance. Section \ref{section_results} is devoted to numerical examples and Section \ref{section_conclu} concludes the paper.

\subsection{Preliminaries and notations}\label{subsec_notations}

The operator $\textnormal{diag}(\cdot)$ creates a diagonal matrix from the elements of a vector. The $\symk^{th}$ row of a matrix $X$ is denoted by $X_\symk$ and the $\symk^{th}$ column by $X_{:\symk}$. For $n$ scalars or row vectors $X_1, \dots, X_n$, $[X_\symh]_{\symh=1}^{n}$ corresponds to the matrix or vector constructed by vertically stacking $X_1, \dots, X_n$. For a vector $x = [x_\symk]_{\symk=1}^{n_B}$ composed of $n_B$ blocks, the vector $x_{-\symhb}$ is the vector of all blocks but the $\symhb^{th}$. The 2-norm of a vector or the spectral norm of a matrix is denoted by $\|\cdot\|$. The norm of a vector $x$ weighted by a positive definite matrix $W$ is denoted by $\|x\|_W = \sqrt{x^\top W x}$. The $\ell_1$ norm is $\|x\|_1$. The Moore-Penrose pseudoinverse of a matrix $X$ is $X^\dagger$. \textcolor{redrev1}{The function $\mt{sgn}(\cdot) \in \{-1,1\}$ gives the sign of a real number and we set $\mt{sgn}(0) = 1$ by definition.}

The expectation and covariance matrix of a random variable $x \in \bb{R}^n$ are denoted by $\bb{E}[x] \in \bb{R}^n$ and $\bb{V}[x] \in \bb{R}^{n \times n}$, respectively. The empirical variance of the elements of a vector $x \in \bb{R}^n$ is given by $\mt{\sffamily{Var}}[x] = \frac{1}{n^2}\lt(n x^\top x \,\textcolor{redrev1}{- (\sum_{i=1}^n x_i)^2}\rt) \in \bb{R} $.

\begin{definition}[Multiaffine function]\label{def_multi}
\textcolor{redrev1}{Let $x_1, \dots, x_{n_B}$ be a decomposition of $x$ into $n_B \geq 1$ blocks, where the blocks  $x_\symh \in \mathbb{R}^{n_\symh}$ are nonempty disjoint subsets of the scalar variables composing $x$ such that their union gives $x$. A scalar function $g(x)$ is called multiaffine if $g(x)$ depends on any single block $x_\symh$ in an affine way.}

A vector field $R(x) = [g_1(x), \dots, g_M(x)]^\top \in \R^M$ is multiaffine if and only if each element $g_\symh$ is multiaffine with respect to the same decomposition $x = [x_1^\top, \dots, x_{n_B}^\top]^\top$.
\end{definition}

\textcolor{redrev1}{As an example of block definitions, the vector $x^\top = [\underbrace{x_{1}, x_{2}}_{x_{b1}}, \underbrace{x_{3}}_{x_{b1}}, \underbrace{x_{4}, x_{5}, x_{6}, x_{7}}_{x_{b3}}]$ can be decomposed into 3 blocks $x_{b1}$, $x_{b2}$, and $x_{b3}$.}

\begin{example}\label{example_multi}
The function $g: \bb R^3 \rightarrow \bb R, g_{\mt{multi}}(x) = x_1x_2x_3 + x_1 - x_3 +1$ is multiaffine with respect to the decomposition $x = [x_1^\top, x_2^\top, x_3^\top]^\top$, while $g_{\mt{sq}}(x) = x_1x_2 - x_2^2$ is not multiaffine due to the squared variable.
\end{example}

Since, by definition, each element of $G$ depends on a single block of variables $x_\symh$ in an affine way, one has the following result.

\begin{corollary}\label{cor_simultaneous}
For a multiaffine vector field $R:\bb R^{\sum_{\symh=1}^{n_B} n_\symh}$ $\rightarrow R^M$, the following $n_B$ equalities simultaneously hold.
\begin{align}\label{aq_phis_equal}
R(x) &= \symC_{\symh}(x_{-\symh}) \!-\! \symF_{\symh}(x_{-\symh}) \cdot \!x_\symh, \quad \forall \symh \in \{1,\dots,n_B\},
\end{align}
where $\symF_\symh(x_{-\symh}) \in \bb R^{M \times n_\symh}$ and $\symC_\symh(x_{-\symh}) \in \bb R^M$ are suitable functions.
\end{corollary}

\noindent
Corollary \ref{cor_simultaneous} provides a trivial linearization with respect to each block $x_\symh$, which is exploited in the sequel to simplify the computations. \textcolor{redrev1}{For example, the function $g_{\mt{multi}}$ from Example \ref{example_multi} can be written as}
\begin{align*}
    \textcolor{redrev1}{\begin{matrix*}[l]
    x_1x_2x_3 +x_1-x_3+1
    &=x_1\underbrace{(x_2x_3+1)}_{F_\symh(x_{-1})} &\;-\, \underbrace{x_3+1}_{C_\symh(x_{-1})}, \\
    &=x_2\underbrace{(x_1x_3)}_{F_\symh(x_{-2})}  &\;+\, \underbrace{x_1 - x_3+1}_{C_\symh(x_{-2})}, \\
    &=x_3\underbrace{(x_1x_2-1)}_{F_\symh(x_{-3})} &\;+\underbrace{x_1+1.\quad\quad}_{C_\symh(x_{-3})}
    \end{matrix*}}
\end{align*}

\begin{definition}
\label{def_gnd}
\textcolor{redrev1}{A standard GND has a density function given by
\begin{align}\label{eq_def_gnd}
    p_\symi(y) = \frac{q_\symi^{\frac{1 + q_\symi}{q_\symi}}}{2 \Gamma(q_\symi^{-1})}e^{- q_\symi |y|^{q_\symi}},
\end{align}
where $\Gamma$ is the Gamma function and $q_\symi \in \bb R_+$ is a positive parameter.}
\end{definition}
\textcolor{redrev1}{The class of standard GNDs includes the standard normal ($q_\symi = 2$) and Laplace ($q_\symi = 1$) distributions, as well as the uniform distribution on $[-1, 1]$ as a limit case for $q_\symi \rightarrow \infty$ \citep{gnd_def_paper}. Note that $q_\symi$ is not limited to integers but can be any positive real number (see Sections \ref{subsec_results_econ} and \ref{subsec_results_envi}).}


\section{Problem statement}\label{section_prob}
The paper focuses on MLE problems where the joint probability distribution has the following structure.

\begin{assumption}\label{ass_multiaffine}
The joint probability density $p(x)$ is proportional to \textcolor{redrev1}{the product of $p_1(r_1(x))$, $\dots$, $p_M(r_M(x))$ of $M$ GNDs with exponents $q_\symi \leq \bar q \in \bb N$.} The functions $r_\symi(x)$ for $\symi = 1, \dots, M$ are multiaffine functions with respect to the same decomposition $x_1, \dots, x_{n_B}$.
\end{assumption}
\begin{example}
The distribution $p(x) \propto e^{-(x_1x_2x_3)^2} \cdot e^{-|x_1 + x_2 x_3|}$ satisfies Assumption \ref{ass_multiaffine}. Indeed, $p(x)$ is proportional to the product of the densities $p_1(g_1(x)) \propto e^{-g_1(x)^2}$ and $p_2(g_2(x)) \propto e^{-|g_2(x)|}$\textcolor{redrev1}{, where} $g_1(x) = x_1x_2x_3$ and $g_2(x) = x_1 + x_2 x_3$ are both multiaffine with respect to the decomposition $x = [x_1^\top, x_2^\top, x_3^\top]^\top$.
\end{example}

\textcolor{redrev1}{Although the limitation to GNDs might seem restrictive, these distributions are quite common in engineering applications and other distributions can often be approximated with their product. For example, the truncated normal distribution is the product between a normal and a uniform distribution, both of which can be obtained from \eqref{eq_def_gnd} with $q_\symi = 2$ and $q_\symi \rightarrow \infty$, respectively. Additionally,} multiaffine models are a generalization of linear ones, and can model complex problems in many fields such as Bayesian and EIV system identification \citep{soderstrom_book, pillonetto_bayesian_outlier_id, ninness_bayesian_sysid_mcmc, chiuso_regularization_systems}, Generalized Kalman smoothing in signal processing \citep{ljung_general_kalman}, and generalized principal component analysis and tensor regressions in machine learning \citep{generalized_pca_vidal, denicolao_bayesian_func_learn, tensor_regressions, tensor_eiv}. In particular, the three following examples have quite simple solutions when the uncertainty is Gaussian-distributed, but can become quite computationally intensive \textcolor{redrev1}{when any GNDs are considered}. The detailed derivations are given in the Appendix.
\begin{example}[Generalized PCA]\label{example_generalized_pca}
Fitting a set of $n$ subspaces defined by their normal vectors $x_\symh$ amounts to solving $\prod_{\symh=\textcolor{redrev1}{1}}^n \phi_\symi^\top x_\symh = \epsilon_h \approx 0$ for all data points $\phi_\symi$ with $\symi = 1, \dots, M$, in the noise-free case. If the residuals $\epsilon_h$ follow a distribution $p$, the most likely fit is given in \citep{generalized_pca_vidal} by
\begin{align*}
    \argmax_{x_{\textcolor{redrev1}{1}}, \dots, x_n} \prod_{\symi=\textcolor{redrev1}{1}}^{M} p \lt( \prod_{\symh=\textcolor{redrev1}{1}}^n \phi_\symi^\top x_\symh\rt).
\end{align*}
\end{example}
\begin{example}[EIV System Identification]\label{example_eiv_sysid}
Let $Z_1 \in \bb R^{T \times n}$ and $Z_2 \in \bb R^{T \times n}$ be the stacked matrices of noisy measurement of the states of a system from $t=0$ to $t=T-1$ and $t=1$ to $t=T$, respectively, and let $X_1$ and $X_2$ be their exact value. From \citep{soderstrom_book}, if the measurements of the states are independent, the maximum-likelihood inference of the matrix $X_0 \in \bb R^{n \times n}$ relating $X_1$ and $X_2$ as $X_2 = X_1 X_0$ can be written as
\begin{align*}
    \argmax_{X_0,X_1} \!\!\prod_{\symi,\symt=\textcolor{redrev1}{1}}^{n,T}\!\! p_{\textcolor{redrev1}{\symt \symi}} \!\!\lt(\!\! Z_{2 \textcolor{redrev1}{\symt \symi}} \!-\! \sum_{\symh = \textcolor{redrev1}{1}}^n X_{1 \textcolor{redrev1}{\symt \symi}}X_{0 \symi \symh} \!\!\rt)\! p_{\textcolor{redrev1}{\symt\unaryminus 1, \symi}}(X_{1 \textcolor{redrev1}{\symt \symi}} \!-\! Z_{1 \textcolor{redrev1}{\symt \symi}}),
\end{align*}
\textcolor{redrev1}{where $p_{\symt \symi}$ is the probability distribution of the measurement error of the state $\symi$ at time $\symt$.}
\end{example}
\begin{example}[Low rank tensor regressions]\label{example_low_rank_tensor}
Linear regression models $Z = \Phi X$ are common in machine learning. However, $Z$ and $\Phi$ may be quite large or high-order tensors in problems such as imaging \citep{tensor_regressions}. It is therefore sometimes necessary to find a lower-rank representation, e.g., in 2 dimensions and with rank 1, $X = x_1  x_2^\top$ where $x_1 \in \bb R^{n_1}$ and $x_2 \in \bb R^{n_2}$. This gives the following maximum likelihood estimation problem
\begin{align*}
    \argmax_{x_1,x_2} \prod_{\symi,\symt=\textcolor{redrev1}{1}}^{n_1,T}p_{\textcolor{redrev1}{\symt \symi}} \lt(Z_{\textcolor{redrev1}{\symt \symi}} - \sum_{\symh = \textcolor{redrev1}{1}}^{n_2} x_{1 \symi}x_{2 \symh} \Phi_{\textcolor{redrev1}{\symt \symh}} \rt),
\end{align*}
where $p_{ht}$ is the probability distribution of the element $\textcolor{redrev1}{(\symt,\symi)}$ of the matrix of residuals $Z-\Phi X$.
\end{example}
More generally, many nonlinear functions can also be expressed as a multiaffine function using a change of variable $x = [\psi_1(v_1), \dots, \psi_{n_B}(v_{n_B})]$, where $v$ is the initial variable and $\psi_\symh$ are nonlinear function. This amounts to learning the coefficients of the basis functions of a model. More examples are provided in Section \ref{subsec_results_envi}.

Using Assumption \ref{ass_multiaffine}, maximizing the probability of the variables $x$ amounts to solve
\begin{align}\label{eq_intro_og_prob}
    \argmax_{x} \prod_{\symi=1}^M p_\symi \lt( \textcolor{redrev1}{r_\symi(x)} \rt),
\end{align}
\textcolor{redrev1}{where the probability densities $p_\symi$ and multiaffine functions $r_\symi$ satisfy Assumption \ref{ass_multiaffine}. Note that since all functions $r_\symi(x)$ are multiaffine with respect to the same decomposition, they can be written as the element of a multiaffine vector field $R(x)$ defined in \eqref{aq_phis_equal}. Moreover, Corollary \ref{cor_simultaneous} allows one to split the problem \eqref{eq_intro_og_prob} into multiple sub-problems in order to build an iterative method and to ensure that the likelihood maximization is tractable. In the sequel, we will focus on optimizing the negative log-likelihood}
\begin{align}\label{eq_lem_multi}
    G(x) = -\sum_{\symi = 1}^M \log p_\symi \lt( r_\symi(x) \rt)\!,\!
\end{align}
\textcolor{redrev1}{instead of \eqref{eq_intro_og_prob}. This operation yields the same optimizer and is commonly done in the MLE literature.}

\textcolor{redrev1}{The problem \eqref{eq_intro_og_prob}} provides the intuition that, similar to more classical regression model fitting problems based on Least Squares, MLE aims to minimize \textcolor{redrev1}{all the scalar} residuals
\begin{subequations}\label{eq_def_residuals}
\begin{align}
    r_{\symi}(x) &= \symf_{\symi 1}(x_{\unaryminus 1}) x_1 - \symc_{\symi 1}(x_{\unaryminus 1}), \\
    &\;\;\vdots \nonumber \\
    &= \symf_{\symi n_B}(x_{\unaryminus n_B}) x_{n_B} - \symc_{\symi n_B}(x_{\unaryminus n_B}),
\end{align}
\end{subequations}
\textcolor{redrev1}{where $\symf_{\symi\symhb}(x_{\unaryminus \symhb})$, and $\symc_{\symi\symhb}(x_{\unaryminus \symhb})$ are the $\symi^{th}$ elements of $\symF_\symhb(x_{\unaryminus \symhb})$ and $\symC_\symhb(x_{\unaryminus \symhb})$, respectively}. \textcolor{redrev1}{Likewise, we define the modified residuals $\hat\rho_\symi(x)$ as}
\begin{align}\label{eq_def_mod_residuals}
    \textcolor{redrev1}{
    \hat\rho_\symi(x) = \mt{sgn}(r_{\symi}(x)) (r_\symi(x)^2 + \alpha)^\frac{1}{\bar q}, 
    }
\end{align}
\textcolor{redrev1}{where $\mt{sgn}(x)$ is the sign function and $\alpha > 0$ is a small real constant. This modification is central to the numerical stability of the AIRLS algorithm presented in the sequel.}

\section{\textcolor{redrev1}{AIRLS Algorithm}}\label{section_method}
Naively, one could solve \eqref{eq_intro_og_prob} by nesting two existing techniques: (i) IRLS, which is very popular in problems with non-Gaussian distribution (e.g., logistic or Laplace) and (ii) Block Coordinate Descent (BCD), which is often used to solve problems with multiaffine costs \citep{irls_convergence, block_relaxation}. Such a nested approach would require IRLS to converge at each iteration of BCD, which would be quite slow for high dimensional problems. To alleviate this high computational complexity, we propose to execute only one iteration of (i) and (ii), alternatingly. The pseudo-code for such an approach is given in Algorithm \ref{alg_airls_def}.

\begin{algorithm}
\caption{AIRLS}\label{alg_airls_def}
\begin{algorithmic}
\Require $\alpha > 0$, \textcolor{redrev1}{$x_{\mt{init}} \in \bb R^n$}
\State $x \gets x_{\mt{init}} ,\; \textcolor{redrev1}{\mc L_-} = \infty ,\; \textcolor{redrev1}{\mc L_+ = -\sum_{\symi=1}^{M} \log p_\symi (\hat\rho_{\symi}(x_{\mt{init}}))}$
\While{$\textcolor{redrev1}{\mc L_-} - \textcolor{redrev1}{\mc L_+} > $ tol}
\For{$\symhb = 1, \dots, n_B$}
    \vspace{-16pt}
    \State \begin{align}\label{eq_w_def}
    &\! W(x) \!=\! \textcolor{redrev1}{\mt{diag}\!\!\lt(\!\!\lt[\!
    \frac{\log\frac{p_1(0)}{p_1(\hat\rho_1(x))}}{|\hat\rho_1(x)|^{\bar q}}, \dots, \frac{\log\frac{p_M(0)}{p_M(\hat\rho_M(x))}}{|\hat\rho_M(x)|^{\bar q}}
    \!\rt]\!\!\rt)}\hspace{-20pt} &&
    \\ \label{eq_abar_def_short}
    & \textcolor{redrev1}{x_\symhb \!\!\leftarrow\! \!\lt(\!\, \symF_{\symhb}(x_{\unaryminus \symhb})^{\!\!\top}\! W(x) \symF_{\symhb}(x_{\unaryminus \symhb}) \!\rt)^{\!\dagger} \!\! \symF_{\symhb}(x_{\unaryminus \symhb})^{\!\!\top}\! W(x) \symC_{\symhb}(x_{\unaryminus \symhb})} \hspace{-14pt} &&
    %
    \end{align}
\vspace{-16pt}
\EndFor
\State $\textcolor{redrev1}{\mc L_-} = \textcolor{redrev1}{\mc L_+} ,\; \textcolor{redrev1}{\mc L_+ = - \sum_{\symi=1}^{M} \log p_\symi (\hat\rho_{\symi}(x))}$
\EndWhile
\end{algorithmic}
\end{algorithm}

In words, Algorithm \ref{alg_airls_def} consists of approximating problem \eqref{eq_intro_og_prob} with the quadratic problem 
\begin{align}\label{eq_abar_def_gen}
    \textcolor{redrev1}{\argmin_{x_i^\prime} \|\symC_{\symhb}(x_{\unaryminus \symhb}) \!-\! \symF_{\symhb}(x_{\unaryminus \symhb})x_i^\prime \|^2_{W(x)}},
\end{align}
w.r.t. one of the variable blocks $x_\symh$, \textcolor{redrev1}{and to use the closed-form solution \eqref{eq_abar_def_short}} to iteratively update $x_\symh$. Next, we address the convergence of Algorithm \ref{alg_airls_def}. Although it belongs to the class of reformulation-linearization algorithms, it is not obvious that the specific implementation for \eqref{eq_intro_og_prob} converges \citep{reformulation_linearization, miguel_nested}.

A main result of this paper is to show that Algorithm \ref{alg_airls_def} converges, as stated in the following theorem. Specifically, the proof of the theorem shows that \textcolor{redrev1}{the following map $\hat G(x)$ decreases} at each iteration. 
\begin{align}\label{eq_lemma_crit}
    \textcolor{redrev1}{\!\!\hat G(x) \!= 
    G(0) +\! \sum_{\symi=1}^M \! \lt( \! \lt( r_\symi(x) \rt)^2 \!+\! \alpha \rt)^{\!\!\frac{q_\symi}{\bar q}}\!\!\!,\!\!}
\end{align}
\textcolor{redrev1}{where $q_\symi \leq \bar q$ is the exponent of $p_\symi$. The weights $W(x)$ are designed such that the quadratic cost in \eqref{eq_abar_def_gen} is equal to $\hat G(x)$ when $x_\symhb^\prime = x_\symhb$ and $\alpha=0$.} The convergence only relies on a small numerical stability parameter $\alpha > 0$, which avoids singularities if some residuals actually decay to zero.


\begin{theorem}[Convergence]\label{thm_abar}
Under Assumption \ref{ass_multiaffine}, Algorithm \ref{alg_airls_def} converges to a fixed point.
\end{theorem}
\begin{proof}
For space reasons, in this proof, we will only write $W_\symhb$, $\symC_{\symhb}$, and $\symF_{\symhb}$ instead of $W(x_\symhb, x_{-\symhb})$, $\symC_{\symhb}(x_{-\symhb})$, and $\symF_{\symhb}(x_{-\symhb})$, respectively. Note that $\alpha > 0$ is required for $W_\symhb$ to be finite, and Assumption \ref{ass_multiaffine} allows one to define $\symC_{\symhb}$ and $\symF_{\symhb}$. The update \eqref{eq_abar_def_short} is equivalent to
\begin{align}\label{eq_abar_def_res}
&\textcolor{redrev1}{
W_\symhb^\frac{1}{2} (\symC_{\symhb} - \symF_{\symhb} x_\symhb^+) = W_\symhb^\frac{1}{2} R(x^+)
}
\\ \nonumber &\quad\quad\;
\textcolor{redrev1}{
= W_\symhb^\frac{1}{2} R(x) 
- W_\symhb^\frac{1}{2} \symF_{\symhb}\lt( \symF_{\symhb}^\top W_\symhb \symF_{\symhb} \rt)^{\dagger} \symF_{\symhb}^\top W_\symhb^\frac{1}{2} W_\symhb^\frac{1}{2} R(x),}
\end{align}
\textcolor{redrev1}{where $R(x)$ is defined in \eqref{aq_phis_equal} and $x$ and $x^+ = [x_{-\symhb}, x_\symhb^+]$ denote the variables before and after the update, respectively.} 
The matrix $P_\symhb = I_{n_\symhb} - \textcolor{redrev1}{W_\symhb^\frac{1}{2} \symF_{\symhb}\lt( \symF_{\symhb}^\top W_\symhb \symF_{\symhb} \rt)^{\dagger} \symF_{\symhb}^\top W_\symhb^\frac{1}{2}}$ is \textcolor{redrev1}{an orthogonal} projection matrix because $P_\symhb^2 = P_\symhb \textcolor{redrev1}{\;= P_\symhb^\top}$. Hence, \textcolor{redrev1}{the norm $\|W_\symhb^\frac{1}{2} R(x)\|_2^2 = \|R(x)\|_{W_\symhb}^2$ does not increase with the update \eqref{eq_abar_def_short}. Note that $W_\symhb \succ 0$ because $\alpha > 0$ and $\mt{sgn}(0) = 1$. Writing this decrease in norm term by term yields}
\begin{align*}
\textcolor{redrev1}{
    \sum_{\symi = 1}^M
    \underbrace{\frac{\log\frac{p_\symi(0)}{p_\symi(\hat\rho_\symi(x))}}{|\hat\rho_\symi(x)|^{\bar q}}}_{w_\symi(x)}
    (r_\symi(x^+))^2 
    \leq \sum_{\symi = 1}^M 
    \frac{\log\frac{p_\symi(0)}{p_\symi(\hat\rho_\symi(x))}}{|\hat\rho_\symi(x)|^{\bar q}} (r_\symi(x))^2.
}
\end{align*}
\textcolor{redrev1}{We can add $w_\symi(x)\alpha $ to both sides of the inequality to transform $(r_\symi(x))^2$ into $|\hat\rho_\symi(x)|^{\bar q}$ and obtain}
\begin{align*}
\textcolor{redrev1}{
    \sum_{\symi = 1}^M\!
    \frac{|\hat\rho_\symi(x^+)|^{\bar q}}{|\hat\rho_\symi(x)|^{\bar q}}  \log\!\frac{p_\symi(0)}{p_\symi(\hat\rho_\symi(x))}
    \!\leq\! \sum_{\symi = 1}^M \!
    \frac{|\hat\rho_\symi(x)|^{\bar q}}{|\hat\rho_\symi(x)|^{\bar q}}  \log\!\frac{p_\symi(0)}{p_\symi(\hat\rho_\symi(x))}.
}
\end{align*}
\textcolor{redrev1}{Since $p_\symi$ is a GND of exponent $q_\symi$, $\log\!\frac{p_\symi(0)}{p_\symi(\hat\rho_\symi(x))} = |\hat\rho(x)|^{q_\symi}$. Hence, the inequality becomes}
\begin{align}\label{eq_proof_conv_norm_dec}
\textcolor{redrev1}{
    \sum_{\symi = 1}^M\!
    |\hat\rho(x^+)|^{\bar q} |\hat\rho(x)|^{q_\symi \!- \bar q}
    \!\leq\! 
    \sum_{\symi = 1}^M \!
    |\hat\rho(x)|^{q_\symi}.
}
\end{align}

\textcolor{redrev1}{We can now relate the left-hand side term to $\hat G(x^+)$ using Young's inequality $ab \leq \frac{q-1}{q} a^\frac{q}{q-1} + \frac{1}{q}b^q$, where $a = 1$, $b = \frac{|\hat\rho(x^+)|^{q_\symi}}{|\hat\rho(x)|^{q_\symi}}$, and $q = \frac{\bar q}{q_\symi} \geq 1$. When both sides are multiplied by $|\hat\rho(x)|^{q_\symi}$, this yields}
\begin{align}\label{eq_proof_conv_young}
\textcolor{redrev1}{
    |\hat\rho(x^+)|^{q_\symi}} &
    \textcolor{redrev1}{\,= 1 \frac{|\hat\rho(x^+)|^{q_\symi}}{|\hat\rho(x)|^{q_\symi}} \cdot |\hat\rho(x)|^{q_\symi}}
    \nonumber\\
    &\!\textcolor{redrev1}{\,
    \leq
    \frac{\bar q \!-\! q_\symi}{\bar q}|\hat\rho(x)|^{q_\symi} \!+\! \frac{q_\symi}{\bar q}
    |\hat\rho(x^+)|^{\bar q} |\hat\rho(x)|^{q_\symi \!- \bar q},
    }
\end{align}
\textcolor{redrev1}{which holds with equality if and only if $|\hat\rho(x)| = |\hat\rho(x^+)|$. One can recognize the terms of the left-hand side of \eqref{eq_proof_conv_norm_dec} on the right-hand side of \eqref{eq_proof_conv_young}. Hence, by summing \eqref{eq_proof_conv_young} over all $\symi = 1, \dots, M$, one obtains that $\sum_{\symi = 1}^M \! |\hat\rho(x^+)|^{q_\symi} \leq \sum_{\symi = 1}^M \! |\hat\rho(x)|^{q_\symi}$, or equivalently, that $\hat G(x)$ decreases at every iteration.}

We continue the proof by showing the existence of a compact positive invariant set of the algorithm. First, we split $x$  as $x^\parallel+x^\perp = [x_\symhb^\parallel]_{\symhb=1}^{n_B} + [x_\symhb^\perp]_{\symhb=1}^{n_B}$, where, for all $\symhb=1,\dots,n_B$, $x_\symhb^\parallel \in \mt{range}(F_\symhb)$ and $x_\symhb^\perp \in \mt{null}(F_\symhb)$. Second, we make the following observations.
\begin{enumerate}[label=(\roman*)]
    \item While $x_{\unaryminus\symhb}$ is not modified by a single update \eqref{eq_abar_def_gen}, any element of both $x^\parallel$ and $x^\perp$ may vary.
    \label{enum_proof_item_0}
    \item In general, $\|x^\parallel\|_2$ is bounded by a constant $x_{\max} < +\infty$ because \textcolor{redrev1}{$\hat G(x)$} is radially unbounded with respect to $x^\parallel$ and must be lesser than \textcolor{redrev1}{$\hat G(x_{\mt{init}})$}.
    \label{enum_proof_item_1}
    \item \textcolor{redrev1}{After an update \eqref{eq_abar_def_short} for any $\symhb$, $x_\symhb^\parallel$ belongs to \eqref{eq_abar_def_gen} and $x_\symhb^\parallel = 0$ because $\mt{null}\!\lt((F_\symhb^\top W_\symhb F_\symhb)^\dagger\rt) = \mt{null}( F_\symhb)$.}
    \label{enum_proof_item_2}
\end{enumerate}
Third, \textcolor{redrev1}{\ref{enum_proof_item_2} implies $\|x_\symhb\|_2 = \|x_\symhb^\parallel\|_2$.} 
Hence, \ref{enum_proof_item_1} implies that $\|x_\symhb\|_2 \leq \|x^\parallel\|_2 \leq x_{\max}$ always holds after the update \eqref{eq_abar_def_short} of $x_\symhb$ for all $\symhb = 1, \dots, n_B$. Fourth and finally, despite \ref{enum_proof_item_0}, the bound on $\|x_\symhb\|_2$ implies that $\|x\|_2 \leq \|x_{\mt{init}}\|_2 + n_B x_{\max}$, which means that the ball of radius $\|x_{\mt{init}}\|_2 + n_B x_{\max}$ is a forward invariant for Algorithm \ref{alg_airls_def}.

To conclude the proof, we consider Algorithm \ref{alg_airls_def} as an autonomous discrete-time system with state $x$, a compact positive invariant set, and a positive-semidefinite function \textcolor{redrev1}{$\hat G(x)$} that decreases over any state trajectory. LaSalle's invariance principle therefore proves the existence a set of accumulation points to which $x$ converges. 
\textcolor{redrev1}{Furthermore, $\hat G(x) = \hat G(x^+)$ can only hold if \eqref{eq_proof_conv_norm_dec} holds with equality, i.e., if $W_\symhb^{-\frac{1}{2}} P_\symhb W_\symhb^\frac{1}{2} R(x) = R(x)$, which is only satisfied if $x_\symhb = \lt( \symF_{\symhb}^\top W_\symhb \symF_{\symhb} \rt)^{\dagger} \symF_{\symhb}^\top W_\symhb \symC_{\symhb}$. For $x$ to be in the set of accumulation points, the equality in must hold for all $\symhb$, meaning that $x$ is a fixed point, which concludes the proof.}
\end{proof}
Algorithm \ref{alg_airls_def} requires a numerical stability parameter $\alpha$ to avoid a division by zero if some residuals are zero. This parameter depends on the machine precision. To tune it, one can start with the baseline floating point precision (e.g., $2.22 \cdot 10^{-16}$ with 64 bits under the IEEE-754 standard), and increase $\alpha$ until the solution stops changing significantly. In what follows, we denote the fixed point to which AIRLS converges by $\hat x$. The next section will discuss the magnitude of the approximation introduced by $\alpha$.

\section{Suboptimality bound \textcolor{redrev1}{for heavy-tailed distributions}}\label{section_suboptimal}
The convergence of Algorithm \ref{alg_airls_def} does not necessarily guarantee that the fixed point is an optimizer of \eqref{eq_intro_og_prob}, due to the approximation introduced by the numerical stability parameter $\alpha$. As $\alpha$ increases, the accuracy of the approximation \eqref{eq_abar_def_gen} of the minimizer of \eqref{eq_lem_multi} decreases. In what follows, we quantify the suboptimality introduced by $\alpha$ \textcolor{redrev1}{when $\bar q = 2$}.

\begin{assumption}\label{ass_optimum}
The problem \eqref{eq_intro_og_prob} has a unique critical point, which is its maximum.
\end{assumption}
Assumption \ref{ass_optimum}, which is used only in this section, is not very restrictive given that the distributions $p_\symi$ are already unimodal. It is verified when the multiaffine functions inside all $p_\symi$ are sufficiently different, i.e., if one $r_{\symh}(x)$ takes the same value at different points, some other $r_{\symk}(x)$ with $\symk \neq \symh$ must take different values on these points. Under Assumption \ref{ass_optimum}, one can characterize the suboptimality of a fixed point $\hat x$ of Algorithm \ref{alg_airls_def} using the map $\hat G(x)$\textcolor{redrev1}{, defined in \eqref{eq_lemma_crit}}. The following Lemma shows that Algorithm \ref{alg_airls_def} actually minimizes $\hat G(x)$ rather than $G(x)$ defined in \eqref{eq_lem_multi}.

\begin{lemma}\label{lem_crit}
\textcolor{redrev1}{Under Assumption \ref{ass_multiaffine}}, any fixed point $\hat x$ of Algorithm \ref{alg_airls_def} is a critical point of the function $\hat G(x)$.
\end{lemma}
\begin{proof}
The proof starts by computing the gradient of each term of \eqref{eq_lemma_crit}. We have
\begin{align}\label{eq_lemma_crit_diff}
    \!\nabla_{x_\symhb} \hat G(x) =  \sum_{\symi=1}^M \symf_{\symi\symhb}(x_{\unaryminus\symhb}) \cdot 2 r_{\symi}(x) \cdot \frac{q_\symi}{\textcolor{redrev1}{\bar q}}\lt( r_{\symi}(x)^2 +\alpha \rt)^{\frac{q_\symi}{\textcolor{redrev1}{\bar q}}-1}\!\!.
\end{align}
If all $p_\symi$ are standard GNDs, the elements of the weight matrix $W(x)$ defined in \eqref{eq_w_def} can be written as
\begin{align}\label{eq_w_def_GND}
\!W_{\symi}(x) \!=\! q_{\symi}((\symf_{\symhb\symi}(x_{\unaryminus \symhb}) \cdot \!x_\symhb \!-\! \symc_{\symhb\symi}(x_{\unaryminus \symhb}))^2 \!+\! \alpha)^{\frac{q_{\symi}}{\textcolor{redrev1}{\bar q}}-1} .
\end{align} 
By simplifying \eqref{eq_lemma_crit_diff} and writing it in matrix form using \eqref{eq_w_def_GND}, we have
\begin{align}\label{eq_lemma_crit_diff_vec}
\nabla_{x_\symhb} \hat G(x) \!=\!  \textcolor{redrev1}{\frac{2}{\bar q}} \symF_{\symhb}(x_{\unaryminus \symhb})^{\!\top} \!W(x) (\symF_{\symhb}(x_{\unaryminus \symhb}) x_\symhb  \!-\! \symC_{\symhb}(x_{\unaryminus \symhb})).
\end{align}
Combining \eqref{eq_lemma_crit_diff_vec} with \eqref{eq_abar_def_gen}, we observe that at a fixed point of Algorithm \ref{alg_airls_def}, $\lt. \nabla_{x_\symhb} \hat G(x) \rt|_{x = \hat x} = \bb 0_{n_\symhb}$ for all $\symhb$, which concludes the proof.
\end{proof}

Since Algorithm \ref{alg_airls_def} optimizes $\hat G$, we are interested to derive bounds on the difference between $G$ and $\hat G$. \textcolor{redrev1}{This can be done for $\bar q = 2$, i.e., when the distributions are heavy-tailed.}

\begin{lemma}\label{lem_ineq}
\textcolor{redrev1}{Under Assumption \ref{ass_multiaffine} and if $\bar q = 2$}, $G(x)$ satisfies
\begin{align}\label{eq_lem_ineq}
    \hat G(x) \geq G(x) \geq \hat G(x) - \sum_{\symi = 1}^{M} \alpha^{\frac{q_{\symi}}{2}},\; \forall x \in \bb R^n.
\end{align}
\end{lemma}
\begin{proof}
The first inequality $\hat G(x) \geq G(x)$ holds because both $\alpha > 0$ and $q_\symi > 0$.  The second inequality is obtained by rewriting $\hat G(x)$ as
\begin{align}\label{eq_f_frac}
    G(0) +\! \sum_{\symi=1}^M \frac{r_{\symi}^2(x)}{\lt( r_{\symi}^2(x) \!+\! \alpha \rt)^{1-\frac{q_\symi}{2}}} + \frac{\alpha}{\lt( r_{\symi}^2(x) \!+\! \alpha \rt)^{1-\frac{q_\symi}{2}}}.
\end{align}
The function $G(x)$ can also be rewritten as $G(x) = G(0) +\! \sum_{\symi=1}^M r_{\symi}^2(x)\lt( r_{\symi}^2(x) \rt)^{\frac{q_\symi}{2}-1}$ if $p_h$ follows Definition \ref{def_gnd}, which is greater than the first term of \eqref{eq_f_frac} because the denominator is smaller as $\alpha$ is positive. Hence,
\begin{align}\label{eq_lem_ineq_proof_fin}
    G(x) \geq \hat G(x) - \alpha \lt( r_{\symi}^2(x) +\alpha \rt)^{\frac{q_\symi}{2}-1}.
\end{align}
Since $r_{\symi}^2(x) \geq 0$, \eqref{eq_lem_ineq} holds true.
\end{proof}

Using both the optimality of the fixed points $\hat x$ for $\hat G$ and the bounds on the difference between $G$ and $\hat G$, we can evaluate the optimality of these fixed points.

\begin{theorem}
Under Assumption \textcolor{redrev1}{\ref{ass_multiaffine} and }\ref{ass_optimum} \textcolor{redrev1}{and if $\bar q = 2$}, the corresponding fixed point $\hat x$ of Algorithm \ref{alg_airls_def} is a $\epsilon$-suboptimal estimate of the minimum of $G(x)$, i.e.,
\begin{align}\label{eq_thm_subopt_main}
    | G(\hat x) - \min_x G(x)| \leq \epsilon = \sum_{\symi = 1}^{M} \alpha^{\frac{q_{\symi}}{2}}.
\end{align}
\end{theorem}
\begin{proof}
Let $x^\star$ be the exact minimum of $G(x)$. Assumption \ref{ass_optimum} implies that, similar to $G$, $\hat G$ also has only one critical point, which is a minimum. Therefore, Lemmas \ref{lem_crit} and \ref{lem_ineq} show that,
\begin{align*}
    G(\hat x)
    \hspace{-10pt}\underbrace{\leq}_{\mt{Lemma \ref{lem_ineq}}} \hspace{-10pt} 
    \hat G(\hat x) 
    \hspace{-10pt}\underbrace{\leq}_{\mt{Lemma \ref{lem_crit}}} \hspace{-10pt}
    \hat G(x^\star)
    \hspace{-10pt}\underbrace{\leq}_{\mt{Lemma \ref{lem_ineq}}} \hspace{-10pt} 
    G(x^\star) + \sum_{\symi = 1}^{M} \alpha^{\frac{q_{\symi}}{2}}.
\end{align*}
Moreover, by definition $G(x^\star) \leq G(\hat x)$. Hence, combining all the inequalities yields
\begin{align*}
    G(x^\star) \leq G(\hat x) + \sum_{\symi = 1}^{M} \alpha^{\frac{q_{\symi}}{2}},
\end{align*}
which implies \eqref{eq_thm_subopt_main} and proves the theorem.
\end{proof}


\section{Error variance estimation}\label{sec_variance}
By definition, the exact solution $x^\star$ of the problem \eqref{eq_intro_og_prob} is the realization of a random variable $x$ such that $r_{\symi}(x) \sim p_\symi$ for all $\symi = 1, \dots, n_B$ and $\symhb = 1, \dots, M$. Hence, any estimate $\hat x_\symhb$ has a statistical error $\hat e_\symhb$, defined as the difference of $\hat x_\symhb$ and its ground truth. As the MLE is unbiased under very weak assumptions \citep{mle_consistency}, the expected error $\bb{E}[e_\symhb]$ is often zero. In this section, we aim to characterize the variance of the estimation error $\hat e_\symhb$.

A common approach to solve this problem is to compute the corresponding Fischer information matrix \citep{proba_ml_book}, which is obtained by differentiating the likelihood defined for the MLE problem. This implies that this method can only be used when all $p_\symi$ are sufficiently smooth.

A more generally applicable method to obtain confidence intervals for point estimates such as the MLE is random resampling (e.g., Jackknife, Bootstrapping) \citep{statistical_inference_book}. For all $\symhb = 1,\dots, n_B$, this approach consists of generating many samples $\hat x_{\unaryminus\symhb}^{(\symkit)},\symkit = 1, \dots, N_S$ of $x_{\unaryminus\symhb}$ according to the distributions $p_\symi$ evaluated at the MLE estimate $\hat x_{\unaryminus\symhb}$, and solve \eqref{eq_lem_multi} $N_S$ times for $x_\symhb$. \textcolor{redrev1}{The variations of the corresponding solutions $\hat x_\symhb^{(\symkit)}$ with respect to the error in $\hat x_{\unaryminus\symhb}$} allow one to evaluate the empirical variance
\begin{align}\label{def_empirical_var}
    \!\!\bb{V}[\hat x_i] =\! \sum_{\symkit = 1}^{N_S} \frac{\hat x_\symhb^{(\symkit)} (\hat x_\symhb^{(\symkit)})^\top}{N_S^2} \!-\! \lt(\sum_{\symkit = 1}^{N_S} \frac{\hat x_\symhb^{(\symkit)}}{N_S} \!\rt) \!\!\! \lt(\sum_{\symkit = 1}^{N_S} \frac{\hat x_\symhb^{(\symkit)}}{N_S} \!\rt)^{\!\!\!\!\top} \!\!,\!\!
\end{align}
which gives $\bb{V}[\hat e_i] = \bb{V}[\hat x_i]$. 
Resampling methods are very popular due to their ease of implementation and reliability. However, as the inference problem must be solved $N_S$ times, a variance estimate is at least $N_S$ times more computationally expensive than the point estimate.

In the following proposition, we exploit the least-squares structure of the iterations of Algorithm \ref{alg_airls_def} to obtain an estimate of $\bb{V}[\hat e_i]$ without solving the inference problem for each sample. We use the compact notations $W_\symhb^{(\symkit)}$, $\symC_{\symhb}^{(\symkit)}$, and $\symF_{\symhb}^{(\symkit)}$ for $W(\hat x_\symhb, x_{\unaryminus\symhb}^{(\symkit)})$, $\symC_{\symhb}(x_{\unaryminus\symhb}^{(\symkit)})$, and $\symF_{\symhb}(x_{\unaryminus\symhb}^{(\symkit)})$, respectively, and denote the \textcolor{redrev1}{empirical conditional expectation $\bb E\Big[ {W_\symhb^{(\symkit)}}^{\!\frac{1}{2}}  (F_\symh^{(\symkit)} \hat x_\symh \!-\!C_\symh^{(\symkit)}) \Big| x_{\unaryminus\symhb}^{(\symkit)} \Big] $ of the weighted} residuals by \textcolor{redrev1}{$\bar{R}^{(\symkit)} = \bb 1 \frac{1}{M} \sum_{\symi=1}^M \sqrt{W_{\symi}(\hat x_\symhb, x_{\unaryminus\symhb}^{(\symkit)})} r_{\symi}(\hat x_\symhb, x_{\unaryminus\symhb}^{(\symkit)})
= \frac{1}{M} \bb 1\bb 1^\top {W_\symhb^{(\symkit)}}^{\!\frac{1}{2}}  (F_\symhb^{(\symkit)} \hat x_\symhb - C_\symhb^{(\symkit)})$, where $\bb 1 \in \bb R^M$ is the vector of all ones}. Additionally, recall that, in general, $\text{\sffamily{Var}} \neq \bb{V}$, as defined in Section \ref{subsec_notations}.

\begin{proposition}\label{prop_var_est}
Let $x_{\unaryminus\symhb}^{(\symkit)}$ with $k = 1,\dots,N_S$ be samples generated from any distribution whose density never vanishes and is centered at the MLE. Under Assumption \ref{ass_multiaffine}, one has
\begin{align}\label{eq_var_upd}
&
\hspace{-4pt}\frac{\!\sum\limits_{\symkit=1}^{N_S} 
\!\! p^{(\symkit)} 
\! \bigg(\!\! 
\textcolor{redrev1}{\big(\sigma_\symhb^{(\symkit)}\big)^2}\! \Big(
\!{\symF_\symhb^{(\symkit)}}^{\!\!\top} 
\! W_\symhb^{(\symkit)}
\! {\symF_\symhb^{(\symkit)}} 
\!\Big)^{\!\!\dagger} 
\!\!+ 
\! {\symF_\symhb^{(\symkit)}}^{\!\ddag} 
\! \bar{R}^{(\symkit)} 
\! {\bar{R}}^{{(\symkit)}^{\!\!{\scriptstyle\top}}}\! 
\! {{\symF_\symhb^{(\symkit)}}^{\!\ddag}}^{\!\top}
 \bigg) \!\!\!}
{\sum_{\symkit=1}^{N_S} p^{(\symkit)}}
\nonumber \\ & \!\!\!\!\!
- \!
\frac{\sum_{\symkit=1}^{N_S} \! p^{(\symkit)} \! {{\symF_\symhb^{(\symkit)}}^{\!\ddag}}^{\!\top}
\!\! \bar{R}_\symhb^{(\symkit)} \!\!}
{\sum_{\symkit=1}^{N_S} p^{(\symkit)}}
\;
\frac{\sum_{\symkit=1}^{N_S} \! p^{(\symkit)} \! {\bar{R}}^{{(\symkit)}^{\!\!{\scriptstyle\top}}}\!\!
{{\symF_\symhb^{(\symkit)}}^{\!\ddag}}^{\!\top} \!\!}
{\sum_{\symkit=1}^{N_S} p^{(\symkit)}}
\! \rightarrow \!
\bb{V}[\hat e_\symhb],\!\!\!
\end{align}
as $N_S \rightarrow \infty$ and where $p^{(\symkit)} = \prod_{\symi=1}^M p_\symi \!\lt(\!r_{\symi}(\hat x_\symhb, x_{\unaryminus\symhb}^{(\symkit)}) \!\rt)$ is the likelihood of a sample, and
\begin{align*}
\textcolor{redrev1}{\big(\sigma_\symhb^{(\symkit)}\big)^2} &= 
\textcolor{redrev1}{\textnormal{\sffamily{Var}}\lt[ {W_\symhb^{(\symkit)}}^{\!\frac{1}{2}} (F_\symh^{(\symkit)} \hat x_\symh \!-\!C_\symh^{(\symkit)})\rt]}
\\
{\symF_\symhb^{(\symkit)}}^{\!\ddag} &=\! \Big({\symF_\symhb^{(\symkit)}}^{\!\top}\! W_\symhb^{(\symkit)} {\symF_\symhb^{(\symkit)}}\Big)^{\!\dagger} \!\!{\symF_\symhb^{(\symkit)}}^\top {W_\symhb^{(\symkit)}}^{\textcolor{redrev1}{\!\frac{1}{2}}}\!.
\end{align*}
\end{proposition}
\begin{proof}
First, we highlight that the variance of $\hat e_\symhb$ is generated by two different factors: the probability distributions $p_\symi$ resulting in nonzero realizations of the residuals $r_{\symi}(\hat x)$, and the uncertainty on the estimates $\hat x_{\unaryminus \symhb}$. Second, we observe that Algorithm \ref{alg_airls_def} relies on the least squares \textcolor{redrev1}{fit of the weighted regression model ${W_\symhb^{(\symkit)}}^{\!\frac{1}{2}}  C_\symh^{(\symkit)} = {W_\symhb^{(\symkit)}}^{\!\frac{1}{2}}  F_\symh^{(\symkit)} \hat x_\symh^{(\symkit)} +\, \varepsilon_\symh^{(\symkit)}\!$. In this setting, the estimation error $\hat e_\symhb^{(\symkit)} = \hat x_\symhb^{(\symkit)} - x_\symhb$ has an expectation and variance conditionned on $x_{\unaryminus\symhb}^{(\symkit)}$ given by}
\textcolor{redrev1}{
\begin{subequations}\label{eq_ls_var_bias}
\begin{align}
    \bb{E}[\hat e_\symhb^{(\symkit)}| x_{\unaryminus\symhb}^{(\symkit)}]
    &= \bb{E}[\hat x_\symhb^{(\symkit)} - \hat x_\symhb| x_{\unaryminus\symhb}^{(\symkit)}] + \underbrace{\bb{E}[\hat x_\symhb - x_\symhb| x_{\unaryminus\symhb}^{(\symkit)}]}_{=0},
    \nonumber \\[-12pt]
    &= \bb E\!\Big[\! \overbrace{{\symF_\symhb^{(\symkit)}}^{\!\ddag} {W_\symhb^{(\symkit)}}^{\!\frac{1}{2}}  (C_\symh^{(\symkit)} \!-\! F_\symh^{(\symkit)} \hat x_\symh)}^{=\hat x_\symhb^{(\symkit)} - \hat x_\symhb}\Big|x_\symhb^{(\symkit)}\!\Big],
    \nonumber \\
    &= -{\symF_\symhb^{(\symkit)}}^{\!\ddag} \bar{R}_\symhb^{(\symkit)},
    \\
    \!\!\bb{V}[\hat e_\symhb^{(\symkit)}| x_{\unaryminus\symhb}^{(\symkit)}]
    &= {{\symF_\symhb^{(\symkit)}}^{\!\ddag}}^{\!\top} \mt{\sffamily{Var}}\Big[ {W_\symhb^{(\symkit)}}^{\!\frac{1}{2}}  (F_\symh^{(\symkit)} \hat x_\symh \!-\!C_\symh^{(\symkit)}) \Big] {\symF_\symhb^{(\symkit)}}^{\!\ddag} \!,
    \nonumber \\
    &= {\big(\sigma_\symhb^{(\symkit)}\big)^2} \lt(\! {\symF_\symhb^{(\symkit)}}^{\!\top} \! W_\symhb^{(\symkit)} \symF_\symhb^{(\symkit)} \!\rt)^{\!\!\dagger} \!\!,\!
\end{align}
\end{subequations}
respectively.}
%
Finally, in order to compute the marginalized variance $\bb{V}[\hat x_\symhb \!-\! x_\symhb]$, one can use the Law of Total Variance (LTV) \citep{ltv_book}
\begin{align}
    \!\! \bb{V}[\hat e_\symhb] \!=\! \bb{E}\!\!\lt[\bb{V}\!\!\lt[\hat e_\symhb| x_{\unaryminus\symhb}^{(\symkit)}\rt]\rt] \!+\! \bb{V}\!\!\lt[\bb{E}\!\!\lt[\hat e_\symhb| x_{\unaryminus\symhb}^{(\symkit)}\rt]\rt].\!\!
\end{align}
The conditional (i.e., inner) variance and expectations is computed using \eqref{eq_ls_var_bias}, and as $N_S \rightarrow \infty$, the marginal (i.e., outer) ones can be obtained by using the empirical formulae
\begin{subequations}\label{eq_empirical_proof}
\begin{align}\label{eq_empirical_exp_proof}
\!\bb{E}\!\!\lt[\bb{V}\!\!\lt[\hat e_\symhb | x_{\unaryminus\symhb}^{(\symkit)}\!\rt]\!\rt] &\!=\! \frac{\sum_{\symkit=1}^{N_S} \!p^{(\symkit)}  \bb{V}\!\!\lt[\hat e_\symhb^{\textcolor{redrev1}{(\symkit)}} | x_{\unaryminus\symhb}^{(\symkit)}\rt]\!}
{\sum_{\symkit=1}^{N_S} \!p^{(\symkit)}},\! 
\\ \label{eq_empirical_var_proof}
\!\bb{V}\!\!\lt[\bb{E}\!\!\lt[\hat e_\symhb| x_{\unaryminus\symhb}^{(\symkit)} \!\rt] \!\rt] &\!=\! 
\frac{\sum_{\symkit=1}^{N_S} p^{(\symkit)} 
\bb{E}\!\!\lt[\hat e_\symhb^{\textcolor{redrev1}{(\symkit)}}| x_{\unaryminus\symhb}^{(\symkit)}\rt] 
\bb{E}\!\!\lt[\hat e_\symhb^{\textcolor{redrev1}{(\symkit)}}| x_{\unaryminus\symhb}^{(\symkit)}\rt]^{\!\top} }
{\sum_{\symkit=1}^{N_S} p^{(\symkit)}} \!\!
\\ \nonumber
&\hspace{-20pt} \!-\!
\frac{\sum_{\symkit=1}^{N_S} \!p^{(\symkit)} 
\bb{E}\!\!\lt[\hat e_\symhb^{\textcolor{redrev1}{(\symkit)}}| x_{\unaryminus\symhb}^{(\symkit)}\rt] }
{\sum_{\symkit=1}^{N_S} \!p^{(\symkit)}} 
\frac{\sum_{\symkit=1}^{N_S} \!p^{(\symkit)} 
\bb{E}\!\!\lt[\hat e_\symhb^{\textcolor{redrev1}{(\symkit)}}| x_{\unaryminus\symhb}^{(\symkit)}\rt]^{\!\!\top} \!\!\!}
{\sum_{\symkit=1}^{N_S} \!p^{(\symkit)}}
.
\end{align}
\end{subequations}
Because $p^{(\symkit)}$ appears in \eqref{eq_empirical_proof}, the empirical variances become exact when the number of samples approaches infinity and their distribution is supported by $\bb R^{\sum_{j\neq i} n_j}$ \citep{sampling_book}. Hence, one has \eqref{eq_empirical_exp_proof} $+$ \eqref{eq_empirical_var_proof} $\rightarrow \bb{V}[\hat e_\symhb]$, which concludes the proof.
\end{proof}

Requiring to compute a pseudoinverse for each sample $\symkit$ can be burdensome when $N_S$ is large. In practice, the samples $x_{\unaryminus\symhb}^{(\symkit)}$ have to be taken very close to the MLE $\hat x_{\unaryminus\symhb}$ (i.e., where $p^{(\symkit)} \not \!\ll \prod_{\symi=1}^M p_\symi( r_\symi(\hat x))$) to avoid having a small denominator in \eqref{eq_var_upd}. Additionally, if the samples are close to the MLE, one has $\symF_\symhb^{(\symkit)} \approx \symF_\symhb(\hat x_{\unaryminus\symhb})$. Thus, ${\symF_\symhb^{(\symkit)}}^{\!\ddag}$ in \eqref{eq_var_upd} can be approximated using the first order Taylor expansion
\begin{align}\label{eq_inverse_approx}
    \big(\!{\symF_\symhb^{(\symkit)}}^{\!\!\top} \! W_\symhb^{(\symkit)} {\symF_\symhb^{(\symkit)}}\big)^{\!\dagger} \approx&\; 
    \big(\symF_\symhb^\top \!\!(\hat x_{\unaryminus\symhb}) W_\symhb {\symF_\symhb}(\hat x_{\unaryminus\symhb}) \big)^{\!\dagger} 
    \\ \nonumber
    &\!\!\!\! \!-\! \Big( \! {\symF_\symhb^{(\symkit)}}^{\!\!\top} \! W_\symhb^{(\symkit)} \! {\symF_\symhb^{(\symkit)}} \!\!-\! \symF_\symhb^\top \!\!(\hat x_{\unaryminus\symhb}) W_\symhb {\symF_\symhb}(\hat x_{\unaryminus\symhb}) \!\Big)\!.
\end{align}

The approximation, \eqref{eq_inverse_approx} improves the computational complexity significantly as only one pseudoinverse is needed for all samples. \textcolor{redrev1}{The accuracy of both \eqref{eq_var_upd} and \eqref{eq_inverse_approx} is discussed in Section \ref{subsec_results_econ}.}

\section{Applications and experiments}\label{section_results}

In this section, we present four applications for AIRLS, which demonstrate the efficacy of AIRLS in terms of speed, scaling and robustness. We begin by considering specific instances of Examples \ref{example_eiv_sysid} and \ref{example_low_rank_tensor} related to practical engineering problems. Moreover, we present two examples from economics and environmental science with more complex likelihoods, which highlights the broadness of problems that AIRLS can address.

\subsection{Online system identification with outliers}\label{subsec_results_sysid}

In \cite{cdc_paper_recursive}, AIRLS is applied to the error-in-variables online system identification problem of estimating the matrices $A \in \bb R^{n_x\times n_x}, B \in \bb R^{n_x \times n_u}$ in the system $x_{\symt+1} = A x_{\symt} + B u_{\symt}$ from noisy measurements of state and control variables
\begin{align}
[\tilde x_{\symt}, \tilde u_{\symt}] = [x_{\symt} + \Delta x_{\symt}, u_{\symt} + \Delta u_{\symt}].
\end{align}
To do so, we define the following data matrices
\begin{align}\label{eq_autocorr_def}
    \!\!\matb 
    \tilde C = \sum_{\symt = 0}^T \beta^{T-\symt} \tilde \Gamma_\symt, \hspace{30pt}\mt{ }
    \\
    \tilde Y = [I, 0_{n_x\times n_z}] \tilde C = E_y \tilde C,
    \\
    \tilde Z = [0_{n_z \times n_x}, I] \tilde C = E_z \tilde C,
    \mate : \tilde \Gamma_\symt = \!\lt[\matb  \tilde x_{\symt+1} \\  \tilde x_{\symt} \\ \tilde  u_{\symt} \mate\rt] \!\!\! \lt[\matb  \tilde x_{\symt+1} \\  \tilde x_{\symt} \\  \tilde u_{\symt} \mate\rt]^{\!\!\top}\!\!\!\!\!,\!\!
\end{align}
where $n_z = n_x+n_u$ and $0 < \beta \leq 1$ is a forgetting factor. Given a prior $\Theta_0$, the identification task consists in estimating the parameters $\Theta = [A, B] \in \R^{n_x \times n_z}$ and the filtered measurements $\hat Z$ such that $\Theta \hat Z - \tilde Y$, $\hat Z - \tilde Z$, and $\Theta - \Theta_0$ are Laplace-distributed and zero-expectation. This distribution provides a good robustness to outliers in the measurements $[\tilde x_{\symt}, \tilde u_{\symt}]$ \citep{cdc_paper_recursive}. The likelihood to maximize is therefore given by
\begin{align}\label{eq_regression_sysid_example_likelihood}
    \mc L(\Theta, Z|\tilde Y, \tilde Z, \Theta_0) = \prod_{\symi=1}^{n_x+n_z}  \prod_{\symj=1}^{n_x}&  e^{-|\Theta_\symj Z_{:\symi} - \tilde Y_{:\symi}|} \\ \nonumber
    & \cdot e^{-\frac{|Z_{:\symi} - \tilde Z_{:\symi}|}{n_x}} \! e^{-\frac{|\Theta_\symj - \Theta_{\symj,0}|}{n_x + n_z}} \!,
\end{align}
which is of the form \eqref{eq_intro_og_prob}. We compare the performance of the MLE of \eqref{eq_regression_sysid_example_likelihood} computed with AIRLS to standard methods based on Gaussian distributions (i.e., subspace identification, Kalman filtering, and recursive total least squares) in Figure \ref{fig:cdc_results}. The plot shows the relative parameter estimation error of the parameters of a two-dimensional system when a varying proportion of the state and input measurements is corrupted by outliers, which are uniformly distributed with a magnitude of 100\% of each state's average value.

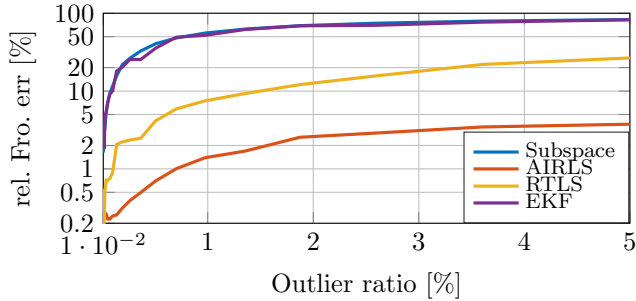
\begin{figure}[ht]
    \centering
%
%
\definecolor{mycolor1}{rgb}{0.00000,0.44700,0.74100}%
\definecolor{mycolor2}{rgb}{0.85000,0.32500,0.09800}%
\definecolor{mycolor3}{rgb}{0.92900,0.69400,0.12500}%
\definecolor{mycolor4}{rgb}{0.49400,0.18400,0.55600}%
\begin{tikzpicture}

\begin{axis}[%
width=0.83\columnwidth,
height=1.1in,
scale only axis,
xmin=0.01,
xmax=5,
xtick={  0.01, 1,   2,   3,  4,  5},
xlabel={Outlier ratio [\%]},
ymode=log,
ymin=0.2,
ymax=100,
ylabel={rel. Fro. err [\%]},
ytick={  0.2, 0.5, 1,   2,   5,  10,  20,  50, 100},
yticklabels={  0.2, 0.5, 1,   2,   5,  10,  20,  50, 100},
yminorticks=true,
axis background/.style={fill=white},
xmajorgrids,
ymajorgrids,
yminorgrids,
legend style={at={(0.684,0)}, anchor=south west, legend cell align=left, align=left, draw=white!15!black, row sep=-0.16cm, font=\scriptsize}
]
\addplot [color=mycolor1, line width=1.2pt]
  table[row sep=crcr]{%
0.01	1.61910051435368\\
0.0138691885653003	2.03359101645307\\
0.0192354391459856	2.80972253691918\\
0.0266779932652033	3.52929626621682\\
0.0370002119138915	5.14651220171146\\
0.0513162915989831	7.07300623334697\\
0.0711715324658232	9.2344696354313\\
0.0987091404249892	11.5766761587765\\
0.136901568167288	15.721174746537\\
0.189871366379743	21.7339429160201\\
0.263336178347187	26.4849203779372\\
0.365225911356268	32.7716586410863\\
0.506538703353372	40.6005929714368\\
0.702528079243062	48.146853777119\\
0.974349440344024	55.6387898925703\\
1.35134361166261	62.5672122521583\\
1.87420393666626	69.2211505514845\\
2.59936878074525	74.5267141660979\\
3.60511357709105	79.2663635040718\\
5	82.9600474415179\\
};
\addlegendentry{Subspace}

\addplot [color=mycolor2, line width=1.2pt]
  table[row sep=crcr]{%
0.01	0.291410112203469\\
0.0138691885653003	0.280375670172788\\
0.0192354391459856	0.266955715636834\\
0.0266779932652033	0.273438727777797\\
0.0370002119138915	0.263220402652873\\
0.0513162915989831	0.22732793680007\\
0.0711715324658232	0.226061504781267\\
0.0987091404249892	0.248228000136179\\
0.136901568167288	0.256001099248256\\
0.189871366379743	0.312388941735475\\
0.263336178347187	0.392219983426649\\
0.365225911356268	0.497251792607906\\
0.506538703353372	0.700422214791405\\
0.702528079243062	1.0054673677974\\
0.974349440344024	1.3910220190101\\
1.35134361166261	1.69061063175565\\
1.87420393666626	2.54809311237077\\
2.59936878074525	2.88972048397072\\
3.60511357709105	3.45756267597291\\
5	3.75213714133917\\
};
\addlegendentry{AIRLS}

\addplot [color=mycolor3, line width=1.2pt]
  table[row sep=crcr]{%
0.01	0.545828268074622\\
0.0138691885653003	0.166137701319153\\
0.0192354391459856	0.253407575221162\\
0.0266779932652033	0.52554244924441\\
0.0370002119138915	0.715666343814436\\
0.0513162915989831	0.71017359850859\\
0.0711715324658232	0.748240580874207\\
0.0987091404249892	0.893349627527041\\
0.136901568167288	2.06411064378967\\
0.189871366379743	2.21884210684622\\
0.263336178347187	2.34800591297986\\
0.365225911356268	2.47486518570053\\
0.506538703353372	4.15161763837868\\
0.702528079243062	5.89953634581419\\
0.974349440344024	7.47793356198546\\
1.35134361166261	9.28056253935202\\
1.87420393666626	12.1097658642429\\
2.59936878074525	15.7020382559705\\
3.60511357709105	21.9812421655024\\
5	26.6514223468153\\
};
\addlegendentry{RTLS}

\addplot [color=mycolor4, line width=1.2pt]
  table[row sep=crcr]{%
0.01	3.20209628090576\\
0.0138691885653003	2.3930207973149\\
0.0192354391459856	1.81155476730571\\
0.0266779932652033	4.88764065516469\\
0.0370002119138915	5.37199190267368\\
0.0513162915989831	6.98128932024942\\
0.0711715324658232	9.02856031002161\\
0.0987091404249892	10.0536816596031\\
0.136901568167288	18.1082040278798\\
0.189871366379743	20.2399996210048\\
0.263336178347187	25.497830279423\\
0.365225911356268	25.4734871181505\\
0.506538703353372	35.5746378471082\\
0.702528079243062	49.2999404411758\\
0.974349440344024	52.0227380169181\\
1.35134361166261	61.833955235916\\
1.87420393666626	68.893413374871\\
2.59936878074525	70.0328179012107\\
3.60511357709105	76.880617737516\\
5	82.3608269849719\\
};
\addlegendentry{EKF}

\end{axis}
\end{tikzpicture}%
    \caption{Relative Frobenius error of the parameter estimates of a double-integrator system using 50 thousand samples for subspace identification \citep{subspace_compa}, AIRLS, Recursive Total Least Squares (RTLS) \citep{rtls_compa}, and the Extended Kalman Filter (EKF) \citep{ekf_compa}. The data has a proportion of outliers up to 5\%. The vertical axis is in log scale.}
    \label{fig:cdc_results}
\end{figure}


\vspace{-8pt}
\subsection{Matrix regression in power systems}\label{subsec_results_pes}
\vspace{-7pt}
The admittance matrix containing all the electrical parameters of a distribution grid is often not known by the operators. Identifying the admittance matrix automatically from voltage and current measurements allows for the optimization of the energy production without requiring too significant investments in modelling. The resulting estimates must however follow some characteristics common to all distribution grids such as sparsity \citep{tomlin}. \cite{paper_tsg} shows that the Bayesian EIV regression of the current on the voltage can produce sufficiently precise Maximum A Posteriori (MAP) estimates. In mathematical terms, this means that one must maximize the likelihood
\vspace{-4pt}
\begin{align}\label{eq_prob_tsg}
    \!\!\mc L(V, Y|\tilde V, \tilde I) =\! \prod_{\symi = 1}^M \! e^{-\|\tilde I_\symi - V Y_\symi\|_2^2} 
    e^{-\|V_\symi - \tilde V_\symi\|_2^2}  e^{-\|Y_\symi\|_1}\!, \!\!
\end{align}
where $V, I \in \R^{M \times N}$ are the nodal voltage and current data matrices containing $N$ samples, their noisy observations are $\tilde V, \tilde I$, and $Y \in \R^{M \times M}$ is the admittance matrix to estimate. Figure \ref{fig:tsg_results} (from \cite{paper_tsg}) shows that the addition of the sparsity promoting prior in the MAP estimate provide a significant improvement over the MLE without this prior, which only uses Gaussian distributions. Moreover, Table \ref{table_tsg} shows that AIRLS computes the MAP estimate \eqref{eq_prob_tsg} significantly faster than other methods adapted to the problem.
\begin{table}[H]
\caption{Comparison of the execution speed of the Block Coordinate Descent (BCD) \citep{S_TLS}, AIRLS, and ADMM \citep{admm_def} to maximize \eqref{eq_prob_tsg} with $M = 9$ and $N = 400$.}
\vspace{4pt}
\centering
\begin{tabular}{|c||c|c|}
    \hline
    Algorithm & iterations to convergence & iterations/second \\
    \hline
    BCD & $\sim$10000 & 1.25 \\
    \hline
    AIRLS & $\sim$10000 & 30 \\
    \hline
    ADMM & $\sim$30000 & 28 \\
    \hline
\end{tabular}
\label{table_tsg}
\end{table}
\vspace{-5pt}
\begin{figure}[H]
\pgfplotsset{
compat=1.11,
legend image code/.code={
\draw[mark repeat=2,mark phase=2]
plot coordinates {
(0cm,0cm)
(0.15cm,0cm)        
(0.3cm,0cm)         
};%
}
}

\begin{tikzpicture}

\begin{axis}[
legend cell align={left},
legend columns=2,
legend style={fill opacity=0.8, draw opacity=1, text opacity=1, at={(1,0.14), font=\small}, anchor=east, draw=white!80!black},
log basis x={10},
log basis y={10},
tick align=outside,
width=8cm,
height=5cm,
tick pos=left,
x grid style={white!69.0196078431373!black},
xmajorgrids,
xmin=1e-05, xmax=0.001,
xmode=log,
xtick style={color=black},
y grid style={white!69.0196078431373!black},
ymajorgrids,
ymin=0.01, ymax=1.22304407455456,
xtick={0.00001,0.00002,0.00005,0.0001,0.0002,0.0005,0.001},
xticklabels={$10^{-5}$,$2\!\cdot\! 10^{-5}$, $5\!\cdot\! 10^{-5}$, $10^{-4}$, $2\!\cdot\! 10^{-4}$, $5\!\cdot\! 10^{-4}$, $10^{-3}$},
ymode=log,
ytick style={color=black},
ytick={0.01,0.02,0.05,0.1,0.2,0.5,1},
yticklabels={1\%,2\%,5\%,10\%,20\%,50\%,100\%},
xlabel={noise to signal ratio},
ylabel={relative estimation error}]
]
\path [draw=black, fill=black, opacity=0.2]
(axis cs:1e-05,0.0489842193570679)
--(axis cs:1e-05,0.0474157806429321)
--(axis cs:2e-05,0.150257670780787)
--(axis cs:5e-05,0.471762255228189)
--(axis cs:0.0001,0.72848141447824)
--(axis cs:0.0002,0.884079015518638)
--(axis cs:0.0005,0.969093798079768)
--(axis cs:0.001,0.989841886116992)
--(axis cs:0.001,0.990158113883008)
--(axis cs:0.001,0.990158113883008)
--(axis cs:0.0005,0.969906201920232)
--(axis cs:0.0002,0.886720984481362)
--(axis cs:0.0001,0.734418585521759)
--(axis cs:5e-05,0.481337744771811)
--(axis cs:2e-05,0.156442329219213)
--(axis cs:1e-05,0.0489842193570679)
--cycle;

\path [draw=fillblue, fill=fillblue, opacity=1.0]
(axis cs:1e-05,0.0181138357147217)
--(axis cs:1e-05,0.0144861642852783)
--(axis cs:2e-05,0.0178133931252681)
--(axis cs:5e-05,0.0324907630408086)
--(axis cs:0.0001,0.0604252717712433)
--(axis cs:0.0002,0.119290537823628)
--(axis cs:0.0005,0.312522590272842)
--(axis cs:0.001,0.781395318220831)
--(axis cs:0.001,0.935454681779169)
--(axis cs:0.001,0.935454681779169)
--(axis cs:0.0005,0.331977409727158)
--(axis cs:0.0002,0.123659462176372)
--(axis cs:0.0001,0.0627247282287567)
--(axis cs:5e-05,0.0342592369591914)
--(axis cs:2e-05,0.0207866068747318)
--(axis cs:1e-05,0.0181138357147217)
--cycle;

\path [draw=red, fill=red, opacity=0.2]
(axis cs:1e-05,0.0585284677683639)
--(axis cs:1e-05,0.0558215322316361)
--(axis cs:2e-05,0.1560751000016)
--(axis cs:5e-05,0.469397341177589)
--(axis cs:0.0001,0.724730151519017)
--(axis cs:0.0002,0.881558992781175)
--(axis cs:0.0005,0.968255033785326)
--(axis cs:0.001,0.989496464289287)
--(axis cs:0.001,0.989853535710714)
--(axis cs:0.001,0.989853535710714)
--(axis cs:0.0005,0.968994966214674)
--(axis cs:0.0002,0.884191007218825)
--(axis cs:0.0001,0.730669848480983)
--(axis cs:5e-05,0.478802658822411)
--(axis cs:2e-05,0.1623248999984)
--(axis cs:1e-05,0.0585284677683639)
--cycle;

\path [draw=fillgreen, fill=fillgreen, opacity=1.0]
(axis cs:1e-05,0.0195454769981818)
--(axis cs:1e-05,0.0154045230018182)
--(axis cs:2e-05,0.0155048763060678)
--(axis cs:5e-05,0.0179358996070273)
--(axis cs:0.0001,0.0246114661349286)
--(axis cs:0.0002,0.0424619603388562)
--(axis cs:0.0005,0.154629095035332)
--(axis cs:0.001,0.493550858982598)
--(axis cs:0.001,0.660899141017402)
--(axis cs:0.001,0.660899141017402)
--(axis cs:0.0005,0.183470904964668)
--(axis cs:0.0002,0.0514380396611438)
--(axis cs:0.0001,0.0280885338650714)
--(axis cs:5e-05,0.0209641003929727)
--(axis cs:2e-05,0.0193451236939322)
--(axis cs:1e-05,0.0195454769981818)
--cycle;

\addplot [semithick, black]
table {%
1e-05 0.0482
2e-05 0.15335
5e-05 0.47655
0.0001 0.73145
0.0002 0.8854
0.0005 0.9695
0.001 0.99
};
\addlegendentry{OLS}
\addplot [semithick, blue]
table {%
1e-05 0.0163
2e-05 0.0193
5e-05 0.033375
0.0001 0.061575
0.0002 0.121475
0.0005 0.32225
0.001 0.858425
};
\addlegendentry{MLE}
\addplot [semithick, red]
table {%
1e-05 0.057175
2e-05 0.1592
5e-05 0.4741
0.0001 0.7277
0.0002 0.882875
0.0005 0.968625
0.001 0.989675
};
\addlegendentry{Lasso}
\addplot [semithick, green!50.1960784313725!black]
table {%
1e-05 0.017475
2e-05 0.017425
5e-05 0.01945
0.0001 0.02635
0.0002 0.04695
0.0005 0.16905
0.001 0.577225
};
\addlegendentry{MAP}
\end{axis}

\end{tikzpicture}

\pgfplotsset{
compat=1.11,
legend image code/.code={
\draw[mark repeat=2,mark phase=2]
plot coordinates {
(0cm,0cm)
(0.3cm,0cm)        
(0.6cm,0cm)         
};%
}
}
\vspace{-5pt}
\caption{Relative estimation error of power grid parameters using Ordinary Least Squares (OLS), Least Absolute Shrinkage and Selection Operator (Lasso), MLE, and MAP.}
\label{fig:tsg_results}
\end{figure}
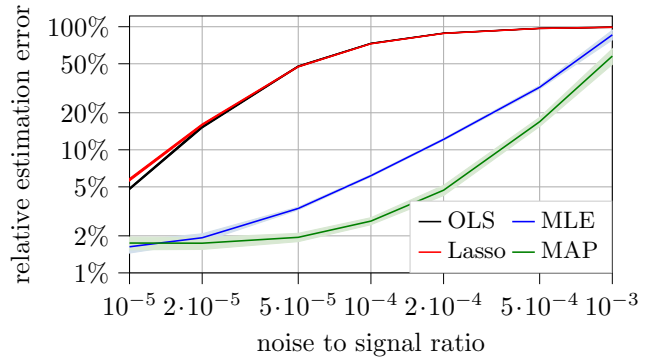
\vspace{-6pt}

\subsection{Economics supply demand problem}\label{subsec_results_econ}

\vspace{-5pt}
In this section, we compare various zeroth order optimization methods for the MLE of a graphical model for a supply-demand problem (e.g., for food harvest \citep{econ_bn_example}). In this model, $S \in  \bb R^T$ is the supply amounts of a good for $\symt = 1, \dots, T$ and $P \in \bb R^{n_T \times T}$ are the prices of $n_T$ different suppliers at each time, which are modified by various taxes or subsidies $\tau \in \bb R^{n_T}$. This all leads to a demand $D \in \bb R^{n_T \times T}$, which corresponds to the quantities of the good that are sold by each seller for a price $P_\symt$. The model is defined by the following distributions, where the value of the parameters are chosen according to \citep{econ_bn_example}. Figure \ref{fig:econ_example} represents this statistical model graphically as a Bayesian network.
\vspace{-5pt}
\begin{subequations}\label{eq_prob_econ_def}
\begin{align}\label{eq_prob_econ_def_a}
    p(S_\symt) &\propto e^{-\lt(\frac{(S_\symt - 100)^2}{200}\rt)^\frac{1}{5}},
    \\ \label{eq_prob_econ_def_c}
    p(P_\symt|S_\symt, \tau) &\propto e^{-\frac{\lt\|(20 - 0.1 S_\symt) \frac{\bb 1 + 0.01\tau}{n_T} - P_\symt \rt\|_2^2}{0.02}}, 
    \\ \label{eq_prob_econ_def_d}
    p(D_\symt | P_\symt) &\propto e^{-\frac{\|200\cdot \bb 1 - 10 P_\symt - D_\symt\|_1}{\sqrt{2}}}.
\end{align}
\end{subequations}
Moreover, $p(\tau)$ is a non-informative prior as defined in \cite{noninformative_priors}.
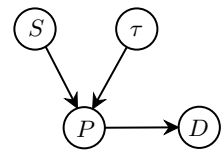
\begin{figure}[H]
    \centering
    \tikzset{every picture/.style={line width=0.75pt}} 

\begin{tikzpicture}[x=0.75pt,y=0.75pt,yscale=-1,xscale=1]

\draw   (129.7,70.6) .. controls (129.7,64.86) and (134.36,60.2) .. (140.1,60.2) .. controls (145.84,60.2) and (150.5,64.86) .. (150.5,70.6) .. controls (150.5,76.34) and (145.84,81) .. (140.1,81) .. controls (134.36,81) and (129.7,76.34) .. (129.7,70.6) -- cycle ;
\draw   (187.7,70.4) .. controls (187.7,64.66) and (192.36,60) .. (198.1,60) .. controls (203.84,60) and (208.5,64.66) .. (208.5,70.4) .. controls (208.5,76.14) and (203.84,80.8) .. (198.1,80.8) .. controls (192.36,80.8) and (187.7,76.14) .. (187.7,70.4) -- cycle ;
\draw    (150.5,70.6) -- (184.7,70.42) ;
\draw [shift={(187.7,70.4)}, rotate = 179.69] [fill={rgb, 255:red, 0; green, 0; blue, 0 }  ][line width=0.08]  [draw opacity=0] (7.14,-3.43) -- (0,0) -- (7.14,3.43) -- (4.74,0) -- cycle    ;
\draw   (104.87,20.77) .. controls (104.87,15.02) and (109.52,10.37) .. (115.27,10.37) .. controls (121.01,10.37) and (125.67,15.02) .. (125.67,20.77) .. controls (125.67,26.51) and (121.01,31.17) .. (115.27,31.17) .. controls (109.52,31.17) and (104.87,26.51) .. (104.87,20.77) -- cycle ;
\draw    (120.67,29.74) -- (135.29,58.07) ;
\draw [shift={(136.67,60.74)}, rotate = 242.7] [fill={rgb, 255:red, 0; green, 0; blue, 0 }  ][line width=0.08]  [draw opacity=0] (7.14,-3.43) -- (0,0) -- (7.14,3.43) -- (4.74,0) -- cycle    ;
\draw   (156.2,20.77) .. controls (156.2,15.02) and (160.86,10.37) .. (166.6,10.37) .. controls (172.34,10.37) and (177,15.02) .. (177,20.77) .. controls (177,26.51) and (172.34,31.17) .. (166.6,31.17) .. controls (160.86,31.17) and (156.2,26.51) .. (156.2,20.77) -- cycle ;
\draw    (161.67,29.74) -- (145.79,58.45) ;
\draw [shift={(144.33,61.07)}, rotate = 298.95] [fill={rgb, 255:red, 0; green, 0; blue, 0 }  ][line width=0.08]  [draw opacity=0] (7.14,-3.43) -- (0,0) -- (7.14,3.43) -- (4.74,0) -- cycle    ;

\draw (134.5,65.33) node [anchor=north west][inner sep=0.75pt]  [font=\footnotesize] [align=left] {$\displaystyle P$};
\draw (191,65.3) node [anchor=north west][inner sep=0.75pt]  [font=\footnotesize] [align=left] {$\displaystyle D$};
\draw (110.33,14.5) node [anchor=north west][inner sep=0.75pt]  [font=\footnotesize] [align=left] {$\displaystyle S$};
\draw (161.67,17.67) node [anchor=north west][inner sep=0.75pt]  [font=\footnotesize] [align=left] {$\displaystyle \tau $};

\end{tikzpicture}
    \caption{Example of a supply demand Bayesian network model with taxes or subsidies.}
    \label{fig:econ_example}
\end{figure}
To generate the data, we sample $S_t$ over the prior distribution \eqref{eq_prob_econ_def_a}, $\tau$ over a normal distribution around 10\% with a standard deviation of 3\%, and use the mode of \eqref{eq_prob_econ_def_c} and \eqref{eq_prob_econ_def_d} for $P_t$ and $D_t$. In order to test Algorithm \ref{alg_airls_def}, we use the realizations of $S_1, \dots, S_T$ and $D_1, \dots, D_T$ to infer the values of $P_1, \dots, P_T$ and $\tau$.

The joint likelihood of a Bayesian network is given by the product of the conditional distributions of each node. Hence, because each distribution follows Assumption \ref{ass_multiaffine}, the joint likelihood
\begin{align} \label{eq_prob_econ}
    \hat P, \hat \tau = \argmin_{P_\symt, \tau} \prod_{\symt = 1}^T \prod_{i \in \{a, b, c\}} (\ref{eq_prob_econ_def}i),
\end{align}
follows \eqref{eq_intro_og_prob}. This experiment is repeated $10$ times. The solid lines in Figures \ref{fig:comparison_accuracy}, \ref{fig:comparison}, \ref{fig:speed}, and \ref{fig:variance_est} represent the average results, while the minimum and maximum are represented by shaded regions.

\emph{Convergence Speed:} We use AIRLS and three benchmark zeroth order algorithms, i.e., discretization \citep{discretization_def}, sampling \citep{sampling_book}, and zeroth order gradient descent (ZOGD) \citep{zero_order_gd}, to solve \eqref{eq_prob_econ}. Figure \ref{fig:comparison_accuracy} shows that, for this multiaffine problem, the convergence of AIRLS is much faster than all other algorithms.

\begin{figure}[H]
    \centering
    \input{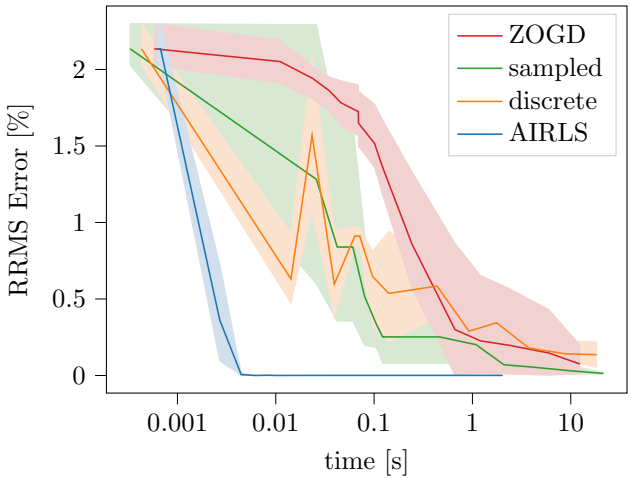}
    \vspace{-15pt}
    \caption{Comparison of convergence speed of four algorithms for the inference of $P_\symt$ in the example \eqref{eq_prob_econ_def} with unknown $\tau$, $T=2$, and $n_T = 1$.}
    \label{fig:comparison_accuracy}
\end{figure}

\emph{Scaling and robustness:}
When $n_T$ or $T$ increase, the sampling and discretization algorithms used in Figure \ref{fig:comparison_accuracy} become very slow and unpractical. We therefore only compare ZOGD to AIRLS for illustrating the scaling of computational time with the dimensionality. Note that each iteration \eqref{eq_abar_def_gen} of AIRLS relies on a least squares problem, which scales with $O(T^3)$. We do not provide a complexity bound for the number of iterations, but the following experiments show that the number of iterations increases only slightly in higher dimensions. Figure \ref{fig:compa_scaling} shows this scaling compared to ZOGD.

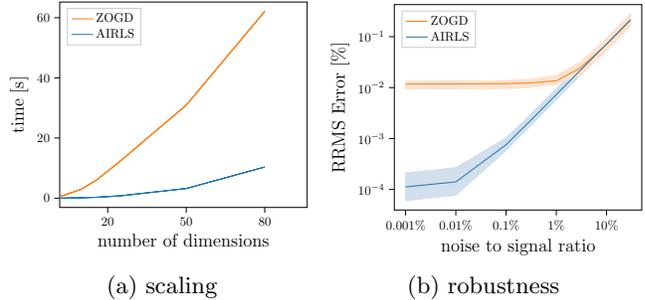
\begin{figure}[H]
     \begin{subfigure}{0.235\textwidth}
\begin{tikzpicture}[scale=0.48]

\definecolor{darkgray176}{RGB}{176,176,176}
\definecolor{darkorange25512714}{RGB}{255,127,14}
\definecolor{lightgray204}{RGB}{204,204,204}
\definecolor{steelblue31119180}{RGB}{31,119,180}

\begin{axis}[
legend cell align={left},
legend style={
  fill opacity=0.8,
  draw opacity=1,
  text opacity=1,
  at={(0.03,0.97)},
  anchor=north west,
  draw=lightgray204
},
log basis x={10},
tick align=outside,
tick pos=left,
xlabel = \large{number of dimensions}, ylabel = \large{time [s]},
x grid style={darkgray176},
xmin=1.66313305803383, xmax=96.2039683037468,
xtick style={color=black},
xtick={20,50,80},
xticklabels={
  \(\displaystyle {20}\),
  \(\displaystyle {50}\),
  \(\displaystyle {80}\)
},
y grid style={darkgray176},
ymin=-3.0848003188769, ymax=65.2185263593992,
ytick style={color=black}
]
\path [draw=steelblue31119180, fill=steelblue31119180, opacity=0.2]
(axis cs:2,0.0198963483174642)
--(axis cs:2,0.0198963483174642)
--(axis cs:4,0.0405352910359701)
--(axis cs:10,0.128075202306112)
--(axis cs:16,0.308940569559733)
--(axis cs:25,0.800132592519124)
--(axis cs:50,3.21003667513529)
--(axis cs:80,10.3569603761037)
--(axis cs:80,10.3569603761037)
--(axis cs:80,10.3569603761037)
--(axis cs:50,3.21003667513529)
--(axis cs:25,0.800132592519124)
--(axis cs:16,0.308940569559733)
--(axis cs:10,0.128075202306112)
--(axis cs:4,0.0405352910359701)
--(axis cs:2,0.0198963483174642)
--cycle;

\path [draw=darkorange25512714, fill=darkorange25512714, opacity=0.2]
(axis cs:2,0.617422342300415)
--(axis cs:2,0.617422342300415)
--(axis cs:4,1.18291433652242)
--(axis cs:10,2.99285324414571)
--(axis cs:16,6.20538544654846)
--(axis cs:25,12.4825989405314)
--(axis cs:50,30.9461341698964)
--(axis cs:80,62.1138296922048)
--(axis cs:80,62.1138296922048)
--(axis cs:80,62.1138296922048)
--(axis cs:50,30.9461341698964)
--(axis cs:25,12.4825989405314)
--(axis cs:16,6.20538544654846)
--(axis cs:10,2.99285324414571)
--(axis cs:4,1.18291433652242)
--(axis cs:2,0.617422342300415)
--cycle;

\addplot [semithick, darkorange25512714, opacity=1.0]
table {%
2 0.617422342300415
4 1.18291433652242
10 2.99285324414571
16 6.20538544654846
25 12.4825989405314
50 30.9461341698964
80 62.1138296922048
};
\addlegendentry{ZOGD}
\addplot [semithick, steelblue31119180, opacity=1.0]
table {%
2 0.0198963483174642
4 0.0405352910359701
10 0.128075202306112
16 0.308940569559733
25 0.800132592519124
50 3.21003667513529
80 10.3569603761037
};
\addlegendentry{AIRLS}
\end{axis}

\end{tikzpicture}
    \caption{scaling}
    \label{fig:compa_scaling}
    \end{subfigure}
     \begin{subfigure}{0.235\textwidth}
\begin{tikzpicture}[scale=0.48]

\definecolor{darkgray176}{RGB}{176,176,176}
\definecolor{darkorange25512714}{RGB}{255,127,14}
\definecolor{lightgray204}{RGB}{204,204,204}
\definecolor{steelblue31119180}{RGB}{31,119,180}

\begin{axis}[
legend cell align={left},
legend style={
  fill opacity=0.8,
  draw opacity=1,
  text opacity=1,
  at={(0.03,0.97)},
  anchor=north west,
  draw=lightgray204
},
log basis x={10},
log basis y={10},
tick align=outside,
tick pos=left,
xlabel = \large{noise to signal ratio}, ylabel = \large{RRMS Error [\%]},
x grid style={darkgray176},
xmin=5.97233193468577e-06, xmax=0.502316353613363,
xmode=log,
xtick style={color=black},
xtick={1e-05,0.0001,0.001,0.01,0.1},
xticklabels={
  \(\displaystyle {0.001\%}\),
  \(\displaystyle {0.01\%}\),
  \(\displaystyle {0.1\%}\),
  \(\displaystyle {1\%}\),
  \(\displaystyle {10\%}\)
},
y grid style={darkgray176},
ymin=3.94734507196118e-07, ymax=0.00425910609660482,
ymode=log,
ytick style={color=black},
ytick={1e-08,1e-07,1e-06,1e-05,0.0001,0.001,0.01,0.1},
yticklabels={
  \(\displaystyle {10^{-6}}\),
  \(\displaystyle {10^{-5}}\),
  \(\displaystyle {10^{-4}}\),
  \(\displaystyle {10^{-3}}\),
  \(\displaystyle {10^{-2}}\),
  \(\displaystyle {10^{-1}}\),
  \(\displaystyle {10^{0}}\),
  \(\displaystyle {10^{1}}\)
}
]
\path [draw=fillblue, fill=fillblue, opacity=1.0]
(axis cs:1e-05,2.10575078337223e-06)
--(axis cs:1e-05,6.02037970238719e-07)
--(axis cs:0.0001,7.71325803354904e-07)
--(axis cs:0.001,5.80936711855702e-06)
--(axis cs:0.003,1.73009688112779e-05)
--(axis cs:0.01,5.76454617234537e-05)
--(axis cs:0.03,0.000172645561546445)
--(axis cs:0.1,0.000564642525088084)
--(axis cs:0.3,0.00169247578810454)
--(axis cs:0.3,0.00279254171538824)
--(axis cs:0.3,0.00279254171538824)
--(axis cs:0.1,0.000937121834529868)
--(axis cs:0.03,0.000283655885766338)
--(axis cs:0.01,9.52269716473669e-05)
--(axis cs:0.003,2.9158835950251e-05)
--(axis cs:0.001,1.02718321096164e-05)
--(axis cs:0.0001,2.67875118522196e-06)
--(axis cs:1e-05,2.10575078337223e-06)
--cycle;

\path [draw=fillorange, fill=fillorange, opacity=1.0]
(axis cs:1e-05,0.000134063693496562)
--(axis cs:1e-05,9.36917958998556e-05)
--(axis cs:0.0001,9.39012769906464e-05)
--(axis cs:0.001,9.59958974160049e-05)
--(axis cs:0.003,0.000100684619221205)
--(axis cs:0.01,0.000116631138456405)
--(axis cs:0.03,0.000176212613635398)
--(axis cs:0.1,0.000517125777113812)
--(axis cs:0.3,0.0016245028623861)
--(axis cs:0.3,0.00276907905864175)
--(axis cs:0.3,0.00276907905864175)
--(axis cs:0.1,0.000915269870371176)
--(axis cs:0.03,0.00031589593051753)
--(axis cs:0.01,0.000174643904935267)
--(axis cs:0.003,0.000146249915357473)
--(axis cs:0.001,0.000138100649973959)
--(axis cs:0.0001,0.000134430334831834)
--(axis cs:1e-05,0.000134063693496562)
--cycle;

\addplot [semithick, darkorange25512714, opacity=1.0]
table {%
1e-05 0.00011789386137594
0.0001 0.000118053838435102
0.001 0.00011965494906855
0.003 0.000123222773184733
0.01 0.000136107813397562
0.03 0.000242320850277456
0.1 0.000703868472576396
0.3 0.00211220616304635
};
\addlegendentry{ZOGD}
\addplot [semithick, steelblue31119180, opacity=1.0]
table {%
1e-05 1.12255363652647e-06
0.0001 1.40915495556433e-06
0.001 7.55592889559831e-06
0.003 2.20730170553992e-05
0.01 7.29069888314777e-05
0.03 0.000217858817854696
0.1 0.000717756433140042
0.3 0.00214091304872781
};
\addlegendentry{AIRLS}
\end{axis}

\end{tikzpicture}
    \vspace{-14pt}
    \caption{robustness}
    \label{fig:compa_robust}
    \end{subfigure}
    
    \vspace{-5pt}
    \caption{Comparison of ZOGD and AIRLS for scaling with $T$, and for robustness to noise. The subfigures show (a) the computation time as a function of $T$ and (b) the RRMS error of the estimate of $P_t$ and $\tau$ depending on the average noise in $S_t$ and $D_t$.}
    \label{fig:comparison}
\end{figure}

In practice, the chosen step size for ZOGD influences both its convergence speed and its accuracy. In this experiment, the step size for ZOGD is $0.99995^\symkit$, where $\symkit$ is the iteration number. This decreasing sequence obtained the best accuracy over all considered noise to signal ratios (see Figure \ref{fig:compa_robust}). In contrast, the accuracy of AIRLS only depends on $\alpha$, which does not (or not directly) influence the convergence speed. This parameter is set to $\alpha = 10^{-3}$ in all experiments.

\emph{Convergence rate:} 
In order to analyze the convergence speed of the algorithm, we increase the number of dimensions to $T=4000$ and $n_T=2$. This higher dimensionality increases the number of iterations required to reach convergence to 13, allowing us to plot the error trajectory with enough resolution in Figure \ref{fig:speed}. Additionally, Figure \ref{fig:speed} shows that the convergence is super-linear, as the error decreases faster than $e_0 \cdot 0.7^{k}$ until it reaches the fixed point, where the error due to $\alpha$ is around 0.005\%.

\begin{figure}[H]
    \centering
\begin{tikzpicture}[scale=1]

\definecolor{lightgray204}{RGB}{204,204,204}
\definecolor{darkgray176}{RGB}{176,176,176}
\definecolor{steelblue31119180}{RGB}{31,119,180}

\begin{axis}[
height=0.24\textwidth,
width=0.45\textwidth,
legend cell align={left},
legend style={
  fill opacity=0.8,
  draw opacity=1,
  text opacity=1,
  at={(0.03,0.03)},
  anchor=south west,
  draw=lightgray204
},
log basis y={10},
tick align=outside,
tick pos=left,
x grid style={darkgray176},
xlabel = {time [s]}, ylabel = {RRMS Error [\%]},
xmin=0.00358251397538636, xmax=217.631238261859,
xtick style={color=black},
xtick={0, 40, 80, 120, 160, 200},
scaled x ticks = false,
xticklabels={
  \(\displaystyle {0}\),
  \(\displaystyle {40}\),
  \(\displaystyle {80}\),
  \(\displaystyle {120}\),
  \(\displaystyle {160}\),
  \(\displaystyle {200}\)
},
y grid style={darkgray176},
ymin=5e-7, ymax=1.374948822,
ymode=log,
ytick style={color=black},
ytick={1e-10,1e-9,1e-08,1e-07,1e-06,1e-05,0.0001,0.001,0.01,0.1,1},
yticklabels={
  \(\displaystyle {10^{-8}}\),
  \(\displaystyle {10^{-7}}\),
  \(\displaystyle {10^{-6}}\),
  \(\displaystyle {10^{-5}}\),
  \(\displaystyle {10^{-4}}\),
  \(\displaystyle {10^{-3}}\),
  \(\displaystyle {10^{-2}}\),
  \(\displaystyle {10^{-1}}\),
  \(\displaystyle {10^{0}}\),
  \(\displaystyle {10^{1}}\),
  \(\displaystyle {10^{2}}\)
}
]
\path [draw=fillblue, fill=fillblue, opacity=1.0]
(axis cs:17.301185409228,1.12438122599957)
--(axis cs:0.0,1.374948822)
--(axis cs:0.0,0.199756077)
--(axis cs:17.301185409228,0.0768535503802802)
--(axis cs:33.7429652134577,0.00933907241863748)
--(axis cs:49.8126209894816,9.645643789627e-05)
--(axis cs:66.6618298610051,6.29053049314692e-06)
--(axis cs:82.8589142004649,2.2760985464524e-06)
--(axis cs:99.4393343687057,9.61627264265052e-08)
--(axis cs:115.960547471046,1.10560475227911e-06)
--(axis cs:132.987022169431,1.09638250234049e-06)
--(axis cs:149.871592275302,1.09752452700365e-06)
--(axis cs:166.267604104678,1.09629009749308e-06)
--(axis cs:182.375651892026,5.21544305834064e-07)
--(axis cs:200.477185360591,1.09639740388139e-06)
--(axis cs:217.631238261859,1.09712618635062e-06)
--(axis cs:217.631238261859,0.000200476576951767)
--(axis cs:217.631238261859,0.000200476576951767)
--(axis cs:200.477185360591,0.000435436904811804)
--(axis cs:182.375651892026,0.000197567585223675)
--(axis cs:166.267604104678,0.000510533756035598)
--(axis cs:149.871592275302,0.00395285226328787)
--(axis cs:132.987022169431,0.0218682844156791)
--(axis cs:115.960547471046,0.0634467417038539)
--(axis cs:99.4393343687057,0.12479808321402)
--(axis cs:82.8589142004649,0.200376136711721)
--(axis cs:66.6618298610051,0.416561403072237)
--(axis cs:49.8126209894816,0.683049781925975)
--(axis cs:33.7429652134577,0.914595124259091)
--(axis cs:17.301185409228,1.12438122599957)
--cycle;

\addplot [semithick, steelblue31119180, opacity=1.0]
table {%
0.0 0.511889728
17.301185409228 0.365091114966686
33.7429652134577 0.24279908742194
49.8126209894816 0.146122577266035
66.6618298610051 0.0707564484505691
82.8589142004649 0.0245624904673306
99.4393343687057 0.00674056245046737
115.960547471046 0.00225173924544692
132.987022169431 0.000782697868670418
149.871592275302 0.000177305680236604
166.267604104678 6.07195744228754e-05
182.375651892026 4.5255309009857e-05
200.477185360591 5.18733656635976e-05
217.631238261859 4.55858279930138e-05
};
\addlegendentry{error of AIRLS}

\addplot [semithick, red, opacity=1.0]
table {%
0.0 0.5
217.631238261859 0.005
}; 
\addlegendentry{$e_k = e_0 \cdot 0.7^k$}
\end{axis}

\end{tikzpicture}
    \vspace{-5pt}
    \caption{Average error trajectory of AIRLS for the problem \eqref{eq_prob_econ_def}, and with $n_T = 2$ and $T = 4000$.}
    \label{fig:speed}
\end{figure}
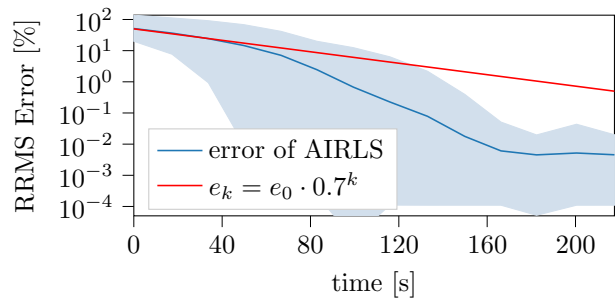

\emph{Variance estimation:} 
We compare the variance computation \eqref{eq_var_upd} and its speed-up \eqref{eq_inverse_approx} for the example \eqref{eq_prob_econ_def} to assess the accuracy of both methods in practice. The baseline used to compare the two estimates is generated by resampling the noise $10$ times and using \eqref{def_empirical_var} on the results. Figure \ref{fig:variance_est} shows that \eqref{eq_var_upd} is quite accurate, while \eqref{eq_inverse_approx} can be conservative for high noise levels.

\begin{figure}[H]
    \centering
\begin{tikzpicture}[scale=0.95]

\definecolor{darkgray176}{RGB}{176,176,176}
\definecolor{darkorange25512714}{RGB}{255,127,14}
\definecolor{forestgreen4416044}{RGB}{44,160,44}
\definecolor{lightgray204}{RGB}{204,204,204}
\definecolor{steelblue31119180}{RGB}{31,119,180}

\begin{axis}[
width=0.48\textwidth,
height=0.36\textwidth,
legend cell align={left},
legend style={
  fill opacity=0.8,
  draw opacity=1,
  text opacity=1,
  at={(0.03,0.97)},
  anchor=north west,
  draw=lightgray204
},
log basis x={10},
log basis y={10},
tick align=outside,
tick pos=left,
xlabel = {noise to signal ratio}, ylabel = {$\|\hat \Sigma\|$},
x grid style={darkgray176},
xmin=1.77827941003892e-05, xmax=5.62341325190349,
xmode=log,
xtick style={color=black},
xtick={0.0001,0.001,0.01,0.1,1},
xticklabels={
  \(\displaystyle {0.001\%}\),
  \(\displaystyle {0.01\%}\),
  \(\displaystyle {0.1\%}\),
  \(\displaystyle {1\%}\),
  \(\displaystyle {10\%}\)
},
y grid style={darkgray176},
ymin=8.33536156949238e-09, ymax=660.679186676286,
ymode=log,
ytick style={color=black},
ytick={1e-11,1e-09,1e-07,1e-05,0.001,0.1,10,1000,100000},
yticklabels={
  \(\displaystyle {10^{-11}}\),
  \(\displaystyle {10^{-9}}\),
  \(\displaystyle {10^{-7}}\),
  \(\displaystyle {10^{-5}}\),
  \(\displaystyle {10^{-3}}\),
  \(\displaystyle {10^{-1}}\),
  \(\displaystyle {10^{1}}\),
  \(\displaystyle {10^{3}}\),
  \(\displaystyle {10^{5}}\)
}
]

\path [draw=fillblue, fill=fillblue, opacity=1.0]
(axis cs:3.16227766016838e-05,5.38586932646094e-05)
--(axis cs:3.16227766016838e-05,2.60817388728525e-08)
--(axis cs:0.0001,5.86437320641909e-08)
--(axis cs:0.000316227766016838,6.60811500943376e-07)
--(axis cs:0.001,7.35480700735254e-06)
--(axis cs:0.00316227766016838,0.000130803617486447)
--(axis cs:0.01,0.00072490775881202)
--(axis cs:0.0316227766016838,0.00650425577316236)
--(axis cs:0.1,0.0538952376172559)
--(axis cs:0.316227766016838,0.0735696608635673)
--(axis cs:1,0.751723540820732)
--(axis cs:3.16227766016838,7.82844313909239)
--(axis cs:3.16227766016838,98.7546292006754)
--(axis cs:3.16227766016838,98.7546292006754)
--(axis cs:1,10.0885677740287)
--(axis cs:0.316227766016838,1.15003846603345)
--(axis cs:0.1,0.37885506272048)
--(axis cs:0.0316227766016838,0.0614505558773961)
--(axis cs:0.01,0.0108053232524022)
--(axis cs:0.00316227766016838,0.000768622488835567)
--(axis cs:0.001,5.68223224239697e-05)
--(axis cs:0.000316227766016838,5.66609305575139e-05)
--(axis cs:0.0001,3.8293320025458e-05)
--(axis cs:3.16227766016838e-05,5.38586932646094e-05)
--cycle;

\path [draw=fillgreen, fill=fillgreen, opacity=1.0]
(axis cs:3.16227766016838e-05,6.34060467741584e-05)
--(axis cs:3.16227766016838e-05,1.00338832306178e-05)
--(axis cs:0.0001,1.51962069411709e-05)
--(axis cs:0.000316227766016838,2.01002665023178e-05)
--(axis cs:0.001,4.88314350003201e-05)
--(axis cs:0.00316227766016838,0.00026671917934671)
--(axis cs:0.01,0.00209527459720252)
--(axis cs:0.0316227766016838,0.0264888648720637)
--(axis cs:0.1,1.8126999328808)
--(axis cs:0.316227766016838,1.15828579545632)
--(axis cs:1,8.79000578399766)
--(axis cs:3.16227766016838,56.7176181392081)
--(axis cs:3.16227766016838,142.662329563143)
--(axis cs:3.16227766016838,142.662329563143)
--(axis cs:1,211.143893788348)
--(axis cs:0.316227766016838,9.26835189296117)
--(axis cs:0.1,102.590528746099)
--(axis cs:0.0316227766016838,0.98787998351204)
--(axis cs:0.01,0.00416192880253018)
--(axis cs:0.00316227766016838,0.000493973838414568)
--(axis cs:0.001,0.000100134934548687)
--(axis cs:0.000316227766016838,7.77870085998582e-05)
--(axis cs:0.0001,6.37960513818699e-05)
--(axis cs:3.16227766016838e-05,6.34060467741584e-05)
--cycle;

\path [draw=fillorange, fill=fillorange, opacity=1.0]
(axis cs:3.16227766016838e-05,6.34059696108493e-05)
--(axis cs:3.16227766016838e-05,1.00339927267525e-05)
--(axis cs:0.0001,1.51940919089205e-05)
--(axis cs:0.000316227766016838,2.0090279104511e-05)
--(axis cs:0.001,4.89923641928694e-05)
--(axis cs:0.00316227766016838,0.000268545408155243)
--(axis cs:0.01,0.00207740296337449)
--(axis cs:0.0316227766016838,0.0177705099074071)
--(axis cs:0.1,0.19007360153697)
--(axis cs:0.316227766016838,1.10740022823306)
--(axis cs:1,7.03710979902715)
--(axis cs:3.16227766016838,52.4584546232392)
--(axis cs:3.16227766016838,95.6137544467895)
--(axis cs:3.16227766016838,95.6137544467895)
--(axis cs:1,11.5283688150863)
--(axis cs:0.316227766016838,1.66370322592412)
--(axis cs:0.1,0.304939791291457)
--(axis cs:0.0316227766016838,0.0462294408605178)
--(axis cs:0.01,0.00417448101507446)
--(axis cs:0.00316227766016838,0.000492807709750962)
--(axis cs:0.001,0.000100037756022182)
--(axis cs:0.000316227766016838,7.77814438538919e-05)
--(axis cs:0.0001,6.37942420849497e-05)
--(axis cs:3.16227766016838e-05,6.34059696108493e-05)
--cycle;

\addplot [semithick, darkorange25512714, opacity=1.0]
table {%
3.16227766016838e-05 2.38364541625326e-05
0.0001 3.14064773800997e-05
0.000316227766016838 4.46961791664447e-05
0.001 6.9913922726716e-05
0.00316227766016838 0.00035397475226012
0.01 0.0031481304822715
0.0316227766016838 0.029597693952082
0.1 0.231166845175033
0.316227766016838 1.34674931244671
1 8.61824695589504
3.16227766016838 72.475355320404
};
\addlegendentry{estimated \eqref{eq_var_upd}}
\addplot [semithick, steelblue31119180, opacity=1.0]
table {%
3.16227766016838e-05 5.09481533145416e-06
0.0001 4.95983465041788e-06
0.000316227766016838 7.10465382415429e-06
0.001 2.9206572938293e-05
0.00316227766016838 0.000370944848984312
0.01 0.00348990914705655
0.0316227766016838 0.0300144912214484
0.1 0.172031048405898
0.316227766016838 0.377198143006394
1 3.45056852830625
3.16227766016838 33.7844929528147
};
\addlegendentry{resampling \eqref{def_empirical_var}}
\addplot [semithick, forestgreen4416044, opacity=1.0]
table {%
3.16227766016838e-05 2.38365477611277e-05
0.0001 3.14070856361856e-05
0.000316227766016838 4.47005102420103e-05
0.001 6.9917905571774e-05
0.00316227766016838 0.000353987268159566
0.01 0.00314219018288634
0.0316227766016838 0.191465663857074
0.1 26.8252202522034
0.316227766016838 2.60679414660856
1 44.3765932842854
3.16227766016838 91.2305233481847
};
\addlegendentry{fast \eqref{eq_inverse_approx}}
\end{axis}

\end{tikzpicture}
    \vspace{-8pt}
    \caption{Comparison of the spectral norms of covariance matrix estimates for various levels of noise, and using resampling \eqref{def_empirical_var}, our estimator \eqref{eq_var_upd}, and its faster approximation \eqref{eq_inverse_approx}.}
    \label{fig:variance_est}
\end{figure}
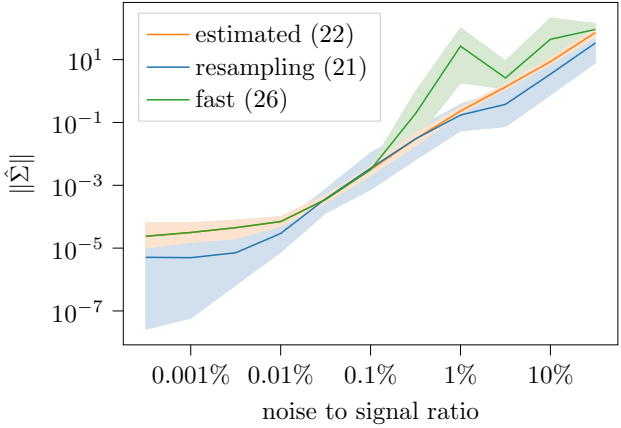

\subsection{A complex example}\label{subsec_results_envi}
As a last example, we describe a problem where all the aforementioned baseline methods fail to provide a meaningful estimate in a reasonable amount of time. We study a simplified model from \citep{watering_example}, which predicts the water usage by farms in a specific area. This model is defined by the following conditional probability distributions \eqref{eq_prob_water_def} and represented in graph form in Figure \ref{fig:water_example}. Inference tasks associated with this model are much more challenging than for \eqref{eq_prob_econ_def} due to the presence of log normal \eqref{eq_prob_water_def_a} and asymmetric Laplace \eqref{eq_prob_water_def_d} densities, as well as numerous relations between the variables. 

In this model, the root nodes $P_\symt$ and $D_\symt$ are: (i) $P_t$ the atmospheric pressure and (ii) $D_t = \sin^2\lt(\frac{\pi}{365} \frac{t}{T}\rt)$, which is a transformation of the day of the year. They influence the normalized sun irradiance $I_\symt$ and the amounts of rain $R_\symt$. Finally, the amount of water $W_\symt$ released by the system depends on the soil humidity $H_\symt$, which depends on the the sun irradiance in both present and past according to an auto-regressive model.

\begin{subequations}\label{eq_prob_water_def}
\begin{align}\label{eq_prob_water_def_a}
    p(P_\symt) &\propto e^{-\frac{\log(P_\symt)^2}{0.02}},
    \\ \label{eq_prob_water_def_c}
    p(I_\symt|D_\symt) &\propto e^{-105|I_\symt - D_\symt - 1| - 100(I_\symt - D_\symt - 1)},
    \\ \label{eq_prob_water_def_d}
    p(R_\symt|P_\symt,D_\symt) &\propto e^{-\frac{|(R_\symt - 3 P_\symt (1 - D_\symt)) + 50(|R_\symt|-R_\symt)|}{3}},
    \\ \label{eq_prob_water_def_e}
    p(H_\symt|I_\symt) &\propto e^{-\frac{\lt( H_\symt - 10 - \sum_{\symkit = 1}^{\symt} 0.9^{\symt-\symkit} I_\symkit \rt)^2}{0.02}},
    \\ \label{eq_prob_water_def_f}
    p(W_\symt|R_\symt, H_\symt) &\propto e^{-\frac{\lt(W_\symt - (H_\symt - R_\symt + 2) \rt)^2}{0.02}}.
\end{align}
\end{subequations}
Moreover, $D_\symt$ is exactly observed so we only introduce a non-informative prior \citep{noninformative_priors} for $p(D_\symt)$.
\begin{figure}[H]
    \centering
    \tikzset{every picture/.style={line width=0.75pt}} 

\begin{tikzpicture}[x=0.75pt,y=0.75pt,yscale=-1,xscale=1]

\draw   (226.03,70.9) .. controls (226.03,65.15) and (230.69,60.5) .. (236.43,60.5) .. controls (242.18,60.5) and (246.83,65.15) .. (246.83,70.9) .. controls (246.83,76.64) and (242.18,81.3) .. (236.43,81.3) .. controls (230.69,81.3) and (226.03,76.64) .. (226.03,70.9) -- cycle ;
\draw   (226.03,121.7) .. controls (226.03,115.95) and (230.69,111.3) .. (236.43,111.3) .. controls (242.18,111.3) and (246.83,115.95) .. (246.83,121.7) .. controls (246.83,127.44) and (242.18,132.1) .. (236.43,132.1) .. controls (230.69,132.1) and (226.03,127.44) .. (226.03,121.7) -- cycle ;
\draw    (236.43,81.3) -- (236.43,108.3) ;
\draw [shift={(236.43,111.3)}, rotate = 270] [fill={rgb, 255:red, 0; green, 0; blue, 0 }  ][line width=0.08]  [draw opacity=0] (7.14,-3.43) -- (0,0) -- (7.14,3.43) -- (4.74,0) -- cycle    ;
\draw   (201.2,21.06) .. controls (201.2,15.32) and (205.86,10.66) .. (211.6,10.66) .. controls (217.34,10.66) and (222,15.32) .. (222,21.06) .. controls (222,26.81) and (217.34,31.46) .. (211.6,31.46) .. controls (205.86,31.46) and (201.2,26.81) .. (201.2,21.06) -- cycle ;
\draw    (217,30.04) -- (231.62,58.37) ;
\draw [shift={(233,61.04)}, rotate = 242.7] [fill={rgb, 255:red, 0; green, 0; blue, 0 }  ][line width=0.08]  [draw opacity=0] (7.14,-3.43) -- (0,0) -- (7.14,3.43) -- (4.74,0) -- cycle    ;
\draw   (252.53,21.06) .. controls (252.53,15.32) and (257.19,10.66) .. (262.93,10.66) .. controls (268.68,10.66) and (273.33,15.32) .. (273.33,21.06) .. controls (273.33,26.81) and (268.68,31.46) .. (262.93,31.46) .. controls (257.19,31.46) and (252.53,26.81) .. (252.53,21.06) -- cycle ;
\draw    (258,30.04) -- (242.12,58.74) ;
\draw [shift={(240.67,61.37)}, rotate = 298.95] [fill={rgb, 255:red, 0; green, 0; blue, 0 }  ][line width=0.08]  [draw opacity=0] (7.14,-3.43) -- (0,0) -- (7.14,3.43) -- (4.74,0) -- cycle    ;
\draw   (278.03,70.9) .. controls (278.03,65.15) and (282.69,60.5) .. (288.43,60.5) .. controls (294.18,60.5) and (298.83,65.15) .. (298.83,70.9) .. controls (298.83,76.64) and (294.18,81.3) .. (288.43,81.3) .. controls (282.69,81.3) and (278.03,76.64) .. (278.03,70.9) -- cycle ;
\draw    (269,29.37) -- (283.62,57.7) ;
\draw [shift={(285,60.37)}, rotate = 242.7] [fill={rgb, 255:red, 0; green, 0; blue, 0 }  ][line width=0.08]  [draw opacity=0] (7.14,-3.43) -- (0,0) -- (7.14,3.43) -- (4.74,0) -- cycle    ;
\draw   (278.37,122.36) .. controls (278.37,116.62) and (283.02,111.96) .. (288.77,111.96) .. controls (294.51,111.96) and (299.17,116.62) .. (299.17,122.36) .. controls (299.17,128.11) and (294.51,132.76) .. (288.77,132.76) .. controls (283.02,132.76) and (278.37,128.11) .. (278.37,122.36) -- cycle ;
\draw    (288.77,81.96) -- (288.77,108.96) ;
\draw [shift={(288.77,111.96)}, rotate = 270] [fill={rgb, 255:red, 0; green, 0; blue, 0 }  ][line width=0.08]  [draw opacity=0] (7.14,-3.43) -- (0,0) -- (7.14,3.43) -- (4.74,0) -- cycle    ;
\draw    (278.37,122.36) -- (249.83,121.76) ;
\draw [shift={(246.83,121.7)}, rotate = 1.21] [fill={rgb, 255:red, 0; green, 0; blue, 0 }  ][line width=0.08]  [draw opacity=0] (7.14,-3.43) -- (0,0) -- (7.14,3.43) -- (4.74,0) -- cycle    ;

\draw (230.33,65.63) node [anchor=north west][inner sep=0.75pt]  [font=\footnotesize] [align=left] {$\displaystyle R$};
\draw (205.67,15.8) node [anchor=north west][inner sep=0.75pt]  [font=\footnotesize] [align=left] {$\displaystyle P$};
\draw (256,14.96) node [anchor=north west][inner sep=0.75pt]  [font=\footnotesize] [align=left] {$\displaystyle D$};
\draw (284,65.63) node [anchor=north west][inner sep=0.75pt]  [font=\footnotesize] [align=left] {$\displaystyle I$};
\draw (283.17,116.93) node [anchor=north west][inner sep=0.75pt]  [font=\footnotesize] [align=left] {H};
\draw (229.83,116.6) node [anchor=north west][inner sep=0.75pt]  [font=\footnotesize] [align=left] {W};

\end{tikzpicture}
    \vspace{-7pt}
    \caption{Example of an environmental Bayesian network model for water use in agriculture.}
    \label{fig:water_example}
\end{figure}
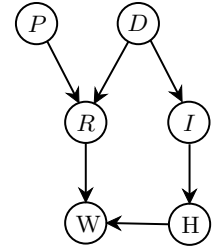

We generate the data by sampling the distributions in \eqref{eq_prob_water_def} for each $\symt=1,\dots, T$. The inference task consists in estimating $R_\symt$ and $I_\symt$ using the values of $P_\symt$, $H_\symt$, $W_\symt$, and $D_\symt$ for all $\symt=1,\dots, T$. The convergence speed and robustness to noisy data are shown in Fig. \ref{fig:comparison_water}, where AIRLS shows a similar favorable performance as in Section \ref{subsec_results_econ}. The comparison with ZOGD is absent because it converges too slowly.

\begin{figure}[H]
     \begin{subfigure}{0.235\textwidth}
\begin{tikzpicture}[scale=0.48]

\definecolor{darkgray176}{RGB}{176,176,176}
\definecolor{steelblue31119180}{RGB}{31,119,180}

\begin{axis}[
log basis x={10},
tick align=outside,
tick pos=left,
xlabel = {\large{number of dimensions}}, ylabel = {\large{time [s]}},
x grid style={darkgray176},
xmin=0, xmax=96.2039683037468,
xtick style={color=black},
xtick={20,50,80},
xticklabels={
  \(\displaystyle {20}\),
  \(\displaystyle {50}\),
  \(\displaystyle {80}\)
},
y grid style={darkgray176},
ymin=0.00915356318155924, ymax=0.564226239522298,
ytick style={color=black}
]
\path [draw=steelblue31119180, fill=steelblue31119180, opacity=0.2]
(axis cs:2,0.0343841393788656)
--(axis cs:2,0.0343841393788656)
--(axis cs:4,0.0387824853261312)
--(axis cs:10,0.055606468518575)
--(axis cs:16,0.0711291948954264)
--(axis cs:25,0.102320432662964)
--(axis cs:50,0.2447420835495)
--(axis cs:80,0.538995663324992)
--(axis cs:80,0.538995663324992)
--(axis cs:80,0.538995663324992)
--(axis cs:50,0.2447420835495)
--(axis cs:25,0.102320432662964)
--(axis cs:16,0.0711291948954264)
--(axis cs:10,0.055606468518575)
--(axis cs:4,0.0387824853261312)
--(axis cs:2,0.0343841393788656)
--cycle;

\addplot [semithick, steelblue31119180, opacity=1.0]
table {%
2 0.0343841393788656
4 0.0387824853261312
10 0.055606468518575
16 0.0711291948954264
25 0.102320432662964
50 0.2447420835495
80 0.538995663324992
};
\end{axis}

\end{tikzpicture}
    \caption{scaling}
    \label{fig:scaling_water}
    \end{subfigure}
     \begin{subfigure}{0.235\textwidth}
\begin{tikzpicture}[scale=0.48]

\definecolor{darkgray176}{RGB}{176,176,176}
\definecolor{steelblue31119180}{RGB}{31,119,180}

\begin{axis}[
log basis x={10},
log basis y={10},
tick align=outside,
tick pos=left,
xlabel = \large{noise to signal ratio [\%]}, ylabel = \large{RRMS Error [\%]},
x grid style={darkgray176},
xmin=4.63533893804407e-14, xmax=4.31467909193264,
xmode=log,
xtick style={color=black},
xtick={1e-16,1e-14,1e-12,1e-10,1e-08,1e-06,0.0001,0.01,1,100,10000},
xticklabels={
  \(\displaystyle {10^{-14}}\),
  \(\displaystyle {10^{-12}}\),
  \(\displaystyle {10^{-10}}\),
  \(\displaystyle {10^{-8}}\),
  \(\displaystyle {10^{-6}}\),
  \(\displaystyle {10^{-4}}\),
  \(\displaystyle {10^{-2}}\),
  \(\displaystyle {10^{0}}\),
  \(\displaystyle {10^{2}}\),
  \(\displaystyle {10^{4}}\),
  \(\displaystyle {10^{6}}\)
},
y grid style={darkgray176},
ymin=7.4255046050952e-15, ymax=1.25407352109709,
ymode=log,
ytick style={color=black},
ytick={1e-19,1e-17,1e-15,1e-13,1e-11,1e-09,1e-07,1e-05,0.001,0.1,10},
yticklabels={
  \(\displaystyle {10^{-17}}\),
  \(\displaystyle {10^{-15}}\),
  \(\displaystyle {10^{-13}}\),
  \(\displaystyle {10^{-11}}\),
  \(\displaystyle {10^{-9}}\),
  \(\displaystyle {10^{-7}}\),
  \(\displaystyle {10^{-5}}\),
  \(\displaystyle {10^{-3}}\),
  \(\displaystyle {10^{-1}}\),
  \(\displaystyle {10^{1}}\),
  \(\displaystyle {10^{3}}\)
}
]
\path [draw=fillblue, fill=fillblue, opacity=1.0]
(axis cs:2e-13,8.2918112679827e-14)
--(axis cs:5e-13,8.201217482713e-14)
--(axis cs:1e-12,1.66255897929896e-13)
--(axis cs:1e-11,1.66093805298809e-12)
--(axis cs:1e-10,1.66138879658691e-11)
--(axis cs:1e-09,1.66144246189106e-10)
--(axis cs:1e-08,1.66145012712357e-09)
--(axis cs:1e-07,1.66144249118863e-08)
--(axis cs:1e-06,1.66137098145649e-07)
--(axis cs:1e-05,1.66070409669202e-06)
--(axis cs:0.0001,1.65499391200231e-05)
--(axis cs:0.001,0.000165834188211538)
--(axis cs:0.01,0.00165839367931608)
--(axis cs:0.1,0.0193958597702849)
--(axis cs:1,0.136231923021226)
--(axis cs:1,0.282887685469923)
--(axis cs:1,0.282887685469923)
--(axis cs:0.1,0.110900142874366)
--(axis cs:0.01,0.0131027262403131)
--(axis cs:0.001,0.0258045482732686)
--(axis cs:0.0001,4.03092295270828e-05)
--(axis cs:1e-05,1.09492523193982e-05)
--(axis cs:1e-06,4.03126365625136e-07)
--(axis cs:1e-07,4.03126753727557e-08)
--(axis cs:1e-08,4.03126711465997e-09)
--(axis cs:1e-09,4.03129696236776e-10)
--(axis cs:1e-10,4.03134414733806e-11)
--(axis cs:1e-11,4.03218014532478e-12)
--(axis cs:1e-12,4.04215549920595e-13)
--(axis cs:5e-13,2.02582395303339e-13)
--(axis cs:2e-13,2.88163382044973e-13)
--cycle;

\addplot [semithick, steelblue31119180, opacity=1.0]
table {%
2e-13 1.52376999221722e-13
5e-13 1.43861404870139e-13
1e-12 2.87244741854088e-13
1e-11 2.87108830652282e-12
1e-10 2.87055504437714e-11
1e-09 2.870491491584e-10
1e-08 2.87048310001036e-09
1e-07 2.8704600319065e-08
1e-06 2.87015185524964e-07
1e-05 3.24196600745932e-06
0.0001 2.97305782565266e-05
0.001 0.00113496534016512
0.01 0.00335949456776592
0.1 0.0540913602606923
1 0.201898685854633
};
\end{axis}

\end{tikzpicture}
    \vspace{-14pt}
    \caption{robustness}
    \label{fig:robust_water}
    \end{subfigure}
    
    \caption{Numerical results of Algorithm \ref{alg_airls_def} applied to the example \eqref{eq_prob_water_def} for scaling with $T$, and for robustness to noise. The subfigures show (a) the computation time as a function of $T$ and (b) the RRMS error of the estimate of $R_t$ and $I_t$ depending on the average noise in the other variables.}
    \label{fig:comparison_water}
\end{figure}
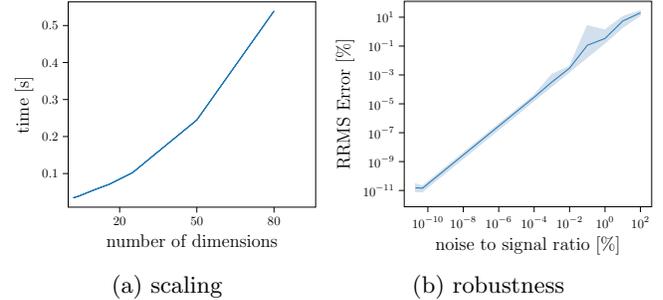




\section{Conclusions}\label{section_conclu}

MLE problems are ubiquitous but can be very challenging to compute when some variables in the problem do not follow well-studied distributions. Moreover, state-of-the-art methods can be slow to solve high-dimensional problems. In this paper, we propose a simple likelihood optimization method, which converges for a wide variety of problems, and produces estimates very close to the optimum. We also provide an algorithm to compute the variance of these estimates efficiently, and show how to apply the whole method to a large class of statistical models. 

While we provide \textcolor{redrev1}{a convergence proof and} an optimality guarantee on the fixed points when all the variables follow GNDs, other classes of distributions will be investigated. Future work on AIRLS will also aim at characterizing the convergence speed analytically, in order to explain the empirically observed super-linear rate.


\bibliographystyle{plainnat}
\bibliography{references}

\section*{Appendix}
\subsection*{Details of Example \ref{example_generalized_pca}}
\cite{generalized_pca_vidal} provides an expression for the fitting problem in the stochastic case. With all points $\phi_\symi$ corrupted by i.i.d. Gaussian noise, the problem is written as
\begin{align*}
    \argmin_{x_{\textcolor{redrev1}{1}}, \dots, x_n} \sum_{\symi=\textcolor{redrev1}{1}}^M\prod_{\symh=\textcolor{redrev1}{1}}^n (\phi_\symi^\top x_\symh)^2 = \argmax_{x_{\textcolor{redrev1}{1}}, \dots, x_n} -\sum_{\symi=\textcolor{redrev1}{1}}^M \prod_{\symh=\textcolor{redrev1}{1}}^n (\phi_\symi^\top x_\symh)^2,
\end{align*}
or equivalently
\begin{align}\label{eq_app_details_generalized_pca}
    \argmax_{x_{\textcolor{redrev1}{1}}, \dots, x_n} \sum_{\symi=\textcolor{redrev1}{1}}^M \ln &\lt( e^{-\lt(\prod_{\symh=\textcolor{redrev1}{1}}^n \phi_\symi^\top x_\symh\rt)^2}\rt)
    \\ \nonumber
    &= \argmax_{x_{\textcolor{redrev1}{1}}, \dots, x_n} \ln \lt( \prod_{\symi=\textcolor{redrev1}{1}}^M e^{-\lt(\prod_{\symh=\textcolor{redrev1}{1}}^n \phi_\symi^\top x_\symh\rt)^2}\rt).
\end{align}
Taking the exponential of \eqref{eq_app_details_generalized_pca} does not change the optimizers $x_{\textcolor{redrev1}{1}}, \dots, x_n$. With $\textcolor{redrev1}{p(\cdot) = \sqrt{\pi^{-1}}e^{-(\cdot)^2}}$, \eqref{eq_app_details_generalized_pca} is therefore equivalent to the problem given in Example \ref{example_generalized_pca}.

\subsection*{Details of Examples \ref{example_eiv_sysid} and \ref{example_low_rank_tensor}}
Both examples are constituted of two matrices of measurements $Z_1$ and $Z_2$, whose noises are assumed to be independent. In Example \ref{example_eiv_sysid}, one tries to fit (i) $Z_2 = X_1 X_0 + \varepsilon_0$ and (ii) $\textcolor{redrev1}{X_1} = Z_1 + \varepsilon_1$, where each element \textcolor{redrev1}{$\varepsilon_{0\symt\symi} = Z_{2\symt\symi} - \sum_{\symh = 1}^n X_{1\symt\symh}X_{0\symh\symi}$ and $\varepsilon_{1\symt\symi} = X_{1\symt\symi} - Z_{1\symt\symi}$} of the matrices $\varepsilon_0$ and $\varepsilon_1$ have density $p_{0\textcolor{redrev1}{\symt \symi}}$ and $p_{1\textcolor{redrev1}{\symt \symi}}$, respectively. In this case, the joint likelihood is equal to
\begin{align*}
    \prod_{\symt = 1}^T \prod_{\symi=1}^n p_{0\textcolor{redrev1}{\symt \symi}}(\varepsilon_{0\textcolor{redrev1}{\symt \symi}})p_{1\textcolor{redrev1}{\symt \symi}}(\varepsilon_{1\textcolor{redrev1}{\symt \symi}})
\end{align*}
Plugging each elements of the matrix regression models (i) and (ii) yields the maximum likelihood shown in Example \ref{example_eiv_sysid}. Note that when all elements of $\varepsilon_0$ and $\varepsilon_1$ are Gaussian and i.i.d., the negative log-likelihood has the well-known form
\begin{align*}
    \|Z_2 - X_1 X_0\|_F^2 + \|X_1 - Z_1\|_F^2.
\end{align*}

Example \ref{example_low_rank_tensor} only contains one regression problem $Z=\Phi X + \varepsilon$. However, as tensors can get quite high dimensional, the parameter $X$ is often constrained to have a certain rank. This can be done by expression $X$ as the product of smaller-sized tensors. For simplicity, consider 2-dimensional tensors, i.e., matrices. If $X \in \bb R^{n\times n}$, any $X$ of rank $r \leq n$ can be expressed as $X = X_1 X_2$, where $X_1 \in \bb R^{n \times r}$ and $X_2 \in \bb R^{r \times n}$. This gives the regression $Z=\Phi X_1 X_2 + \varepsilon$, where each element of the noise $\varepsilon_{\textcolor{redrev1}{\symt \symi}}$ is distributed according to the density $p_{\textcolor{redrev1}{\symt \symi}}$. Plugging this regression model element-wise into \eqref{eq_intro_og_prob} yields the MLE problem given in Example \ref{example_low_rank_tensor}.

\end{document}